\documentclass[letterpaper]{article} 
\usepackage{aaai2026}  
\usepackage{times}  
\usepackage{helvet}  
\usepackage{courier}  
\usepackage[hyphens]{url}  
\usepackage{placeins}
\usepackage{graphicx} 
\usepackage{natbib}  
\usepackage{caption} 
\usepackage{algorithm}
\usepackage{algorithmic}
\usepackage{subfig}
\usepackage{newfloat}
\usepackage{listings}
\DeclareCaptionStyle{ruled}{labelfont=normalfont,labelsep=colon,strut=off} 
\floatstyle{ruled}
\newfloat{listing}{tb}{lst}{}
\floatname{listing}{Listing}
\usepackage{graphicx}
\graphicspath{{images/}}

\usepackage{cite}
\usepackage{graphicx}
\usepackage{epstopdf}
\usepackage{algorithm,algorithmic}
\usepackage{booktabs}
\usepackage{tabularx}
\usepackage{makecell}

\usepackage{amssymb}
\usepackage{amsmath}
\usepackage{amsthm}
\newtheorem{lemma}{Lemma}[section]
\newtheorem{rem}{CustomTheorem}[section]
\newtheorem{Remark}[rem]{Remark}
\newtheorem{defi}{CustomTheorem}
\newtheorem{Definition}[defi]{Definition}
\title{MDND: Unsupervised Learning Guided by Non-Differentiable Refinement for
Shape Correspondence}
\author{
 Qinsong Li\textsuperscript{\rm 1}\thanks{Equal contribution.}, 
 Jing Meng\textsuperscript{\rm 2}\footnotemark[1],
 Haibo Wang\textsuperscript{\rm 2}, 
 Shengjun Liu\textsuperscript{\rm 2}\thanks{Corresponding author.}
}
\affiliations{
 \textsuperscript{\rm 1}Big Data Institute, Central South University, Changsha, Hunan 410083, China\\
    \textsuperscript{\rm 2}School of Mathematics and Statistics, Central South University, Changsha, Hunan 410083, China\\
    qinsli.cg@csu.edu.cn, wykqhmj1112@163.com, wang\_haibo2017@163.com, shjliu.cg@csu.edu.cn
}

\usepackage{bibentry}

\begin{document}

\maketitle

\begin{abstract}
	Deep functional map frameworks (DFM) for shape correspondence are powerful, yet fundamentally limited by their reliance on end-to-end differentiability. This constraint prevents the integration of highly accurate, non-differentiable refinement techniques, capping their overall performance, especially on challenging non-isometric shapes. To overcome this, we introduce MDND, a novel DFM paradigm built on the principle of merging differentiable and non-differentiable components. Our framework facilitates unsupervised learning guided by an internal, non-differentiable refinement. Specifically, MDND employs a dual-branch architecture: a non-differentiable refinement branch leverages a novel, multiscale iterative solver to produce highly robust correspondences, acting as a refined target. Concurrently, a fully differentiable branch learns to predict correspondences from features. The entire system is trained end-to-end without supervision by enforcing a consistency loss that compels the differentiable branch to learn from the superior, refined results of the non-differentiable branch. Extensive experiments show that MDND sets a new state-of-the-art, demonstrating remarkable robustness on shapes with non-isometric deformations and topological noise. 
    
\end{abstract}


\begin{links}
    \link{Code}{https://github.com/AMAWDBAC/MDND}
 
\end{links}

\section{Introduction}
Establishing correspondences between non-rigid 3D shapes is a fundamental problem in computer vision and graphics, with broad applications in texture transfer \cite{dinh2005texture}, shape interpolation \cite{Ezuz2019}, statistical shape analysis \cite{bogo2014faust}, and animation \cite{sumner2004deformation}. Despite decades of research, the problem remains highly challenging, especially when shapes undergo severe non-isometric deformations or contain substantial topological noise.

\begin{figure}[h!t]
	\begin{centering}
		\includegraphics[width=1.0\linewidth]{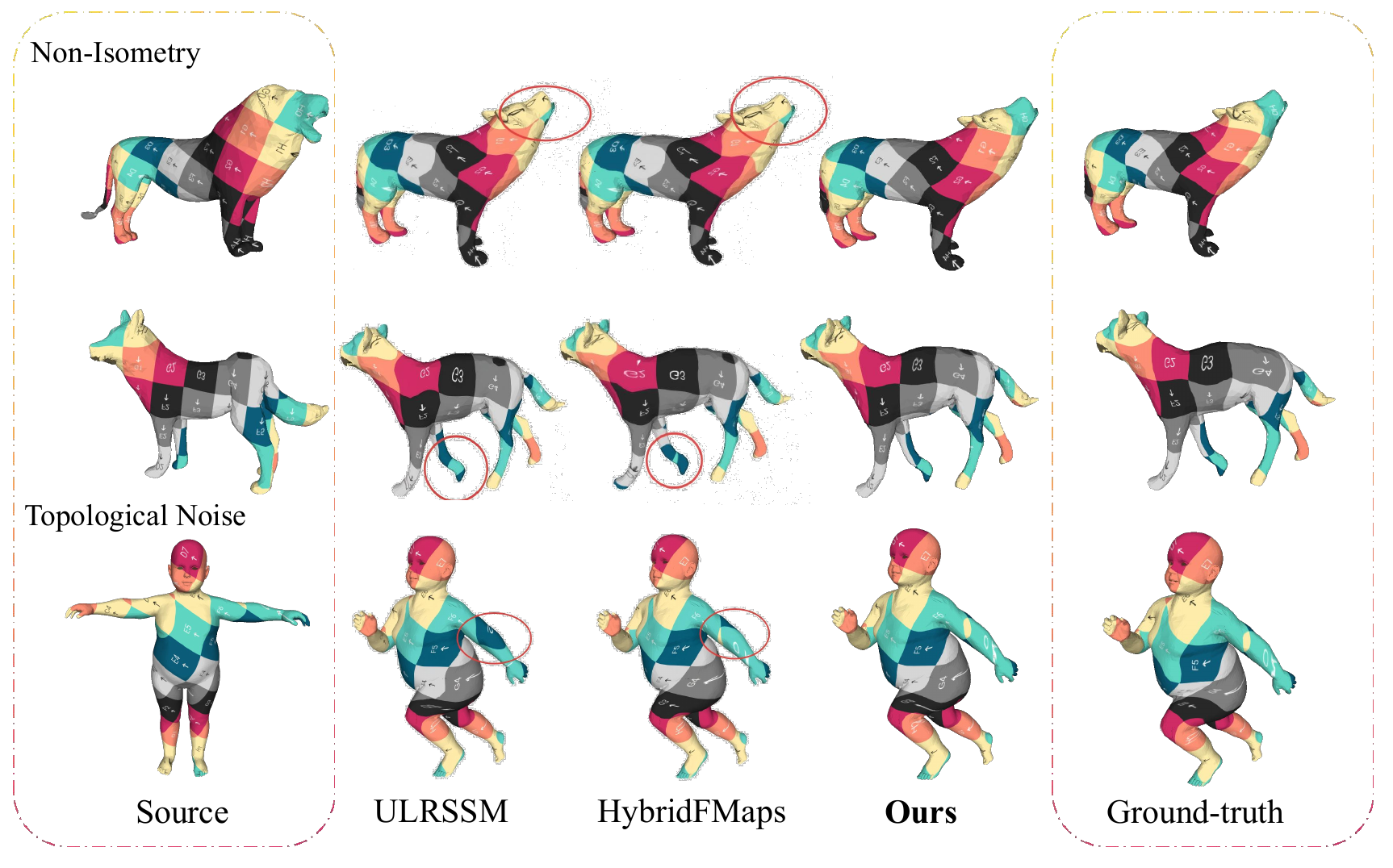}
	\end{centering}
	\caption{Qualitative comparison on challenging shapes. We visualize correspondences via texture transfer, comparing our method against state-of-the-art approaches like ULRSSM~\cite{Cao2023Unsupervised} and HybridFMaps~\cite{bastian2024hybrid}. Our method produces noticeably more accurate and coherent maps in challenging scenarios involving non-isometric deformations (top rows) and significant topological noise (bottom row).}
	\label{shouyetu}
\end{figure}
The functional map framework has emerged as a dominant paradigm for non-rigid shape correspondence, revolutionizing the field by recasting the problem from matching points to matching functions between shapes \cite{Ovsjanikov2012, Pai2021, liu2022incremental}. This approach represents the correspondence compactly as a small matrix, the functional map, in a spectral basis typically composed of the eigenfunctions of the Laplace-Beltrami Operator (LBO). The advent of deep learning has given rise to deep functional map (DFM) pipelines \cite{litany2017deep, Donati2020, donati2022deep, vigano2025nam}, which learn optimized feature descriptors directly from data to improve matching accuracy. Over time, the field has converged on a powerful, yet highly structured, standard architecture. These pipelines typically employ a Siamese-style network to extract pointwise features, which are then fed into a differentiable layer to solve for the functional map matrix. A key innovation was the introduction of a second, parallel branch that promotes consistency between the functional map and a soft pointwise map derived from feature similarity \cite{Cao2023Unsupervised,luo2025deep}. This two-branch architecture, which enforces consistency between the functional and spatial domains, has become the de facto standard for achieving state-of-the-art results in both supervised and unsupervised settings.

This convergence on a standard pipeline, however, exposes a fundamental and shared limitation: an implicit reliance on \textit{end-to-end differentiability}. This constraint is a significant bottleneck, as it precludes the integration of powerful axiomatic refinement methods. These methods \cite{Melzi2019, eisenberger2020smooth, Ling2021, li2024deformable}, often formulated as iterative optimization procedures, can produce highly accurate correspondences from a reasonable initialization. However, they frequently incorporate complex, non-differentiable operations such as nearest-neighbor searches or discrete optimization steps \cite{Ren2021}, making them incompatible with gradient-based training. Recent efforts to create \textit{differentiable} versions of these axiomatic algorithms have been met with significant trade-offs. Such approaches often require storing and differentiating through large, dense soft-correspondence matrices, leading to quadratic memory complexity that is infeasible for high-resolution meshes \cite{Eisenberger2020a, Li2022LearningMF, hu2023rfmnet}. Furthermore, differentiating through the linear system solvers inherent to many DFM pipelines is known to be numerically unstable \cite{Donati2020}. This suggests a flawed premise in the current research trajectory: by forcing the axiomatic method to become differentiable, its original robustness and elegance are often compromised in favor of a fragile, inefficient, and incomplete approximation that fits the deep learning mold.

In this work, we challenge the necessity of end-to-end differentiability and propose a novel DFM paradigm, termed MDND (\textbf{M}erging \textbf{D}ifferentiable and \textbf{N}on-\textbf{D}ifferentiable components). Instead of forcing axiomatic methods into a differentiable framework, we leverage their full, uncompromised power by treating them as non-differentiable supervisory oracles. In our proposed framework, a deep network first predicts an initial correspondence. This map is then refined by a powerful, off-the-shelf axiomatic algorithm, which may contain non-differentiable steps. The resulting high-quality map from this oracle is then used as a pseudo-ground-truth target. A consistency loss between the network's initial prediction and the oracle's refined output is backpropagated through the feature-learning network. This process effectively teaches the network to generate features that produce better initializations—ones that the axiomatic method can readily refine to a high-accuracy solution, thereby dramatically improving learning efficiency and final matching precision.
A powerful learning framework requires an equally powerful oracle. However, many efficient and popular axiomatic refiners, such as ZoomOut \cite{Melzi2019} or MWP \cite{Ling2021}, are built upon the LBO eigenbasis. While the LBO basis is intrinsically defined and thus robust to isometries, it fundamentally struggles to characterize the high-frequency, extrinsic details like bending and creasing that define non-isometric deformations. This makes LBO-based refiners inherently ill-suited for the most challenging matching problems. Inspired by recent work showing that a hybrid basis \cite{bastian2024hybrid}—combining the intrinsic LBO basis with an extrinsic basis derived from an elastic thin-shell energy (ELA) \cite{hartwig2023elastic}—is far more expressive for non-isometric shapes, we propose a new iterative refinement by generalizing the principles of state-of-the-art refiners MWP \cite{Ling2021} to this more powerful hybrid basis. Our refiner is theoretically grounded, efficient, and highly robust to severe non-isometric deformations and topological artifacts. In summary, the main contributions of this work include:
\begin{itemize}
	\item We propose \textbf{MDND}, the first approach to integrate \textit{non-differentiable iterative refinement} into the deep functional map framework, aiming to enhance feature learning and matching accuracy.
	
	\item We introduce an effective refinement with theoretical justification, and seamlessly incorporate it as a \textit{supervisory oracle} within the MDND framework, significantly improving robustness in challenging scenarios. 
	
	\item Extensive experiments across a wide range of challenging conditions demonstrate that our method sets a new state-of-the-art, particularly in cases involving substantial non-isometric deformations and topological noise.
\end{itemize}

\section{Related Work}
\label{sec:formatting}
Shape correspondence is a central topic in geometric processing, and we refer the reader to recent surveys for a comprehensive overview \cite{Sahillioglu2020, deng2022survey}. Our work builds upon three key pillars of research: axiomatic functional maps, deep functional map methods, and the development of expressive spectral bases.

\textbf{Axiomatic Functional Maps and Refinement.} The functional map framework, introduced by Ovsjanikov et al. \cite{Ovsjanikov2012}, provides an elegant algebraic representation of correspondences. This foundational work sparked a wave of axiomatic (i.e., non-learning-based) methods aimed at improving map quality. These approaches typically focus on designing sophisticated energy functions to enforce desirable properties, such as orientation preservation \cite{Ren2018, donati2022complex}, consistency with geometric wavelets \cite{Ling2021, liu2022incremental}, or multi-shape consistency \cite{huang2020, Gao2021}. To optimize these energies, powerful iterative refinement strategies have become standard practice. Methods like ZoomOut \cite{Melzi2019}, MWP \cite{Ling2021}, and Smoothshells \cite{eisenberger2020smooth} iteratively alternate between solving for a functional map and updating a pointwise correspondence. While these axiomatic techniques can achieve high accuracy, their performance is fundamentally limited by their reliance on hand-crafted feature descriptors \cite{Aubry2011, Sun2009, Salti2014, liu2024awedd} and spectral bases of LBO, which often fail in the presence of strong non-isometric deformations.

\textbf{Deep Functional Maps.} To mitigate the reliance on hand-crafted features, the field has shifted towards Deep Functional Maps (DFM). FMNet \cite{litany2017deep} was the first to learn feature descriptors for functional maps in a supervised manner. Unsupervised learning was subsequently introduced, using losses based on geodesic distances \cite{halimi2019unsupervised} or structural properties of the functional map \cite{roufosse2019unsupervised}. The DFM pipeline has progressively matured with architectural innovations. GeomFmaps \cite{Donati2020} introduced a differentiable regularized map solver, DUOFMNet \cite{donati2022deep}, which learned orientation-aware features using complex functional maps \cite{donati2022complex}. AttentiveFMaps \cite{Li2022LearningMF} employed spectral attention to handle varying resolutions. A significant breakthrough came with the introduction of dual-branch architectures \cite{Cao2023Unsupervised, cao2024revisiting, sun2023spatially, luo2025deep}, which enforce consistency between the functional map domain and the pointwise spatial domain. However, a common thread unites these advanced pipelines: their complete reliance on end-to-end differentiability prevents them from incorporating the powerful, non-differentiable solvers developed in the axiomatic literature. Compared to the above works, which only optimize descriptors, recent research has attempted to use generative models to directly learn functional mappings simultaneously \cite{zhuravlev2025denoising, emery2025}. However, this work requires training on large labeled datasets and is less effective than the former.

\textbf{Intrinsic-Extrinsic Bases in Functional Maps.} 
\label{ie}The functional map framework is built upon spectral bases. The eigenvectors of  LBO are the conventional choice due to their intrinsic nature, which provides robustness to isometric deformations. However, this very property is a limitation in non-isometric settings, where crucial extrinsic information (e.g., bending and creases) is lost. To address this, recent work has explored more expressive bases. Hartwig et al. \cite{hartwig2023elastic} introduced an extrinsic basis derived from the Hessian of a thin-shell elastic energy (the ELA-basis), which is highly sensitive to such fine-grained details. Building on this, HybridFMaps \cite{bastian2024hybrid} demonstrated that a hybrid spectral space combining the intrinsic LBO basis and the extrinsic ELA basis is significantly more expressive for non-isometric shapes. Our work is motivated by these advancements, and we leverage a hybrid basis to construct a refinement oracle that is robust to the challenging deformations where purely intrinsic methods fail.

\section{Background}
In this section, we provide an overview of the background knowledge related to several key modules integral to our approach.
\subsection{Functional Map} 
\label{FM}
Let $T: \mathcal{M} \to \mathcal{N}$ be a pointwise map from shape $\mathcal{M}$ to shape $\mathcal{N}$. The induced functional map $T_F: \mathcal{L}^2 (\mathcal{N}) \to \mathcal{L}^2 (\mathcal{M})$ transforms square-integrable real-valued functions from $\mathcal{N}$ to $\mathcal{M}$. Specifically, for any function $f_\mathcal{N} \in \mathcal{L}^2 (\mathcal{N})$, the corresponding function $f_\mathcal{M} \in \mathcal{L}^2 (\mathcal{M})$ is defined by the composition $f_\mathcal{M} = T_F(f_\mathcal{N}) = f_\mathcal{N} \circ T$. Assuming that $\mathcal{L}^2 (\mathcal{M})$ and $\mathcal{L}^2 (\mathcal{N})$ are equipped with basis functions $\{\phi_i^{\mathcal{M}}\}_{i\geq1}$ and $\{\phi_j^{\mathcal{N}}\}_{j\geq1}$, respectively, the functional map can be represented as a matrix $\mathbf{C} = (c_{ij})$, where $c_{ij} = \langle T_F(\phi_j^\mathcal{N}), \phi_i^\mathcal{M} \rangle$. Each element of $\mathbf{C}$ captures the relationship between the two sets of basis functions.

In the discrete setting, shapes $\mathcal{M}$ and $\mathcal{N}$ are typically represented as triangular meshes with $m$ and $n$ vertices, respectively. The pointwise map $T$ is denoted by $\Pi_{\mathcal{MN}}\in \mathbb{R}^{m \times n}$, where $\Pi_{\mathcal{MN}}(i,j)=1$ if $T(i)=j$, and $0$ otherwise. Here, $i$ and $j$ represent vertex indices on shape $\mathcal{M}$ and $\mathcal{N}$, respectively. And we regard it as a proper pointwise map. Let $\Phi^{\mathrm{LBO}}_\mathcal{M} \in \mathbb{R}^{m \times k}$ and $\Phi^{\mathrm{LBO}}_\mathcal{N} \in \mathbb{R}^{n \times k}$ denote the matrices containing the first $k$ discretized Laplacian eigenfunctions for each shape. The functional map $\mathbf{C^{\mathrm{LBO}}_{\mathcal{NM}}}$ is given by the projection of $\Pi_{\mathcal{MN}}$ onto the corresponding functional basis:
\begin{equation}\label{eq:P2C}
	\mathbf{C^{\mathrm{LBO}}_{\mathcal{NM}}}=(\Phi^{\mathrm{LBO}}_{\mathcal{M}})^{\dagger}\Pi_{\mathcal{MN}}\Phi^{\mathrm{LBO}}_{\mathcal{N}},
\end{equation}
where $\dagger$ denotes the Moore-Penrose pseudo-inverse. Since the functional map $\mathbf{C^{\mathrm{LBO}}_{\mathcal{NM}}}$ in Eq.\eqref{eq:P2C} arises from a pointwise correspondece, we call it \textit{proper functional map}~\cite{Ren2021}.

When the pointwise map (and, by extension, the functional map) is unknown, the functional map can be computed by solving the following optimization problem: 
\begin{equation} \label{eq:C}
	\mathbf{C^{\mathrm{LBO}}_{\mathcal{NM}}} = \underset{\mathbf{C}}{\arg\min} E_{\mathrm{desc}}\left ( \mathbf{C} \right ) +\alpha E_{\mathrm{reg}}\left ( \mathbf{C} \right ), 
\end{equation}
where $E_{\mathrm{desc}}\left ( \mathbf{C} \right )=\left\|\mathbf{C^{\mathrm{LBO}}_{\mathcal{NM}}  (\Phi^{\mathrm{LBO}}_{\mathcal{N}})^{\dagger} \mathbf{F}_\mathcal{N} -(\Phi^{\mathrm{LBO}}_{\mathcal{M}} )^{\dagger}\mathbf{F}_\mathcal{M} }\right\|_\mathrm{F}^2 $ enforces descriptor preservation. Here, $\mathbf{F}_\mathcal{M} \in \mathbb{R}^{m \times d}$ and $\mathbf{F}_\mathcal{N} \in \mathbb{R}^{n \times d}$ are $d$-dimensional feature matrices for $\mathcal{M}$ and $\mathcal{N}$, respectively.  The term $E_{\mathrm{reg}} \left ( \mathbf{C} \right )$ represents a regularization function that promotes structural consistency in $\mathbf{C}$, and $\alpha$ is a regularization parameter. Once $\mathbf{C^{\mathrm{LBO}}_{\mathcal{NM}}}$ is obtained, the pointwise map is derived through a nearest neighbor search in the spectral embedding spaces ($\Phi^{\mathrm{LBO}}_{\mathcal{N}}$ and $\Phi^{\mathrm{LBO}}_{\mathcal{M}}\mathbf{C^{\mathrm{LBO}}_{\mathcal{NM}}}$). However, the quality of the resulting map is often suboptimal. To improve this, various refinement techniques have been proposed \cite{Melzi2019, magnet2022smooth, donati2022complex, Ren2018, Ling2021}, which iteratively alternate between optimizing the functional and pointwise maps to enhance accuracy.

\begin{figure*}[htbp]
	\centering
	\subfloat[{\normalfont(a) Typical dual-branch design}]{
		\includegraphics[scale=0.27]{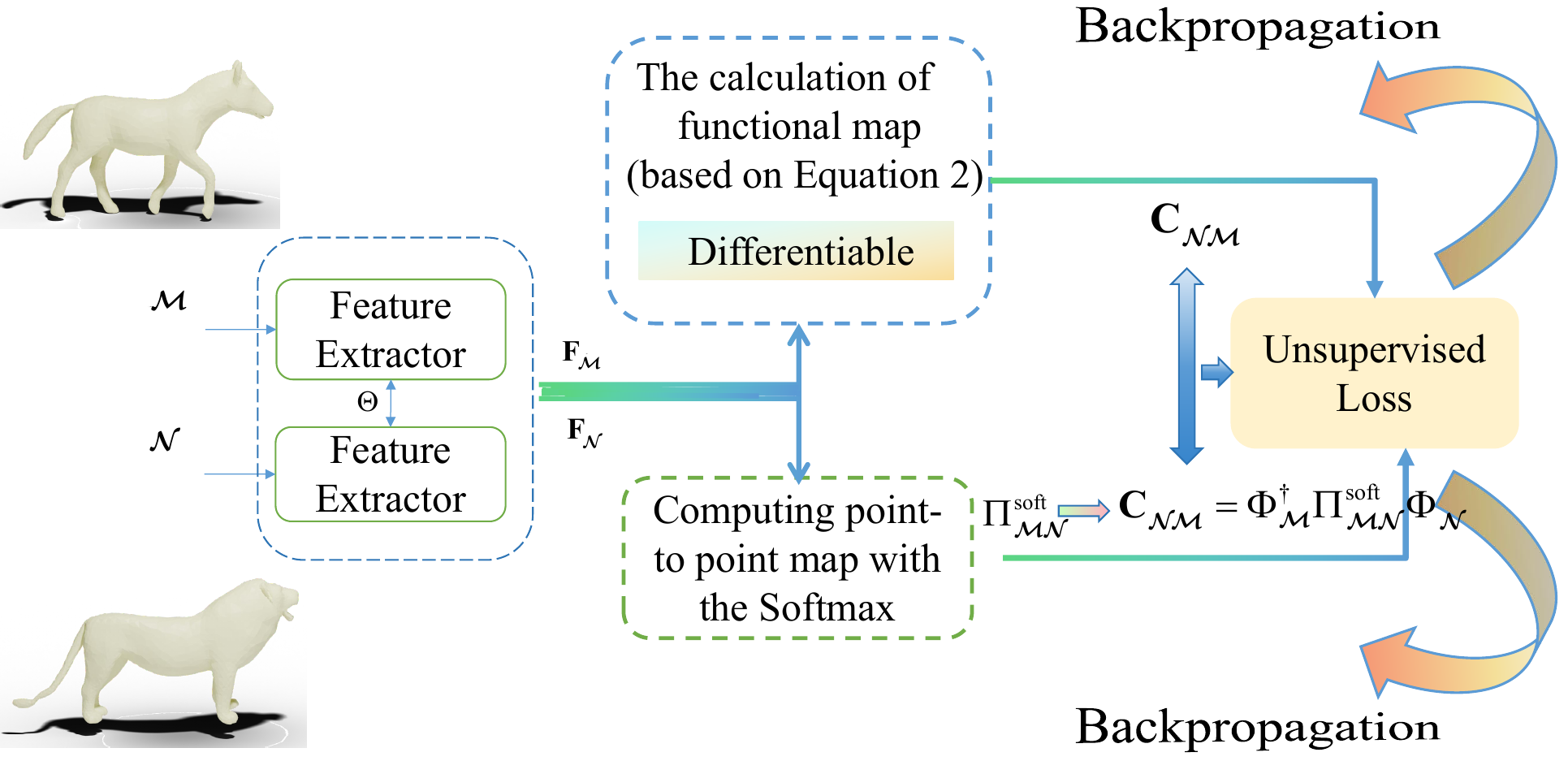}
	}
	\hfill
	\subfloat[{\normalfont (b) Our proposed MDND}]{
		\includegraphics[scale=0.27]{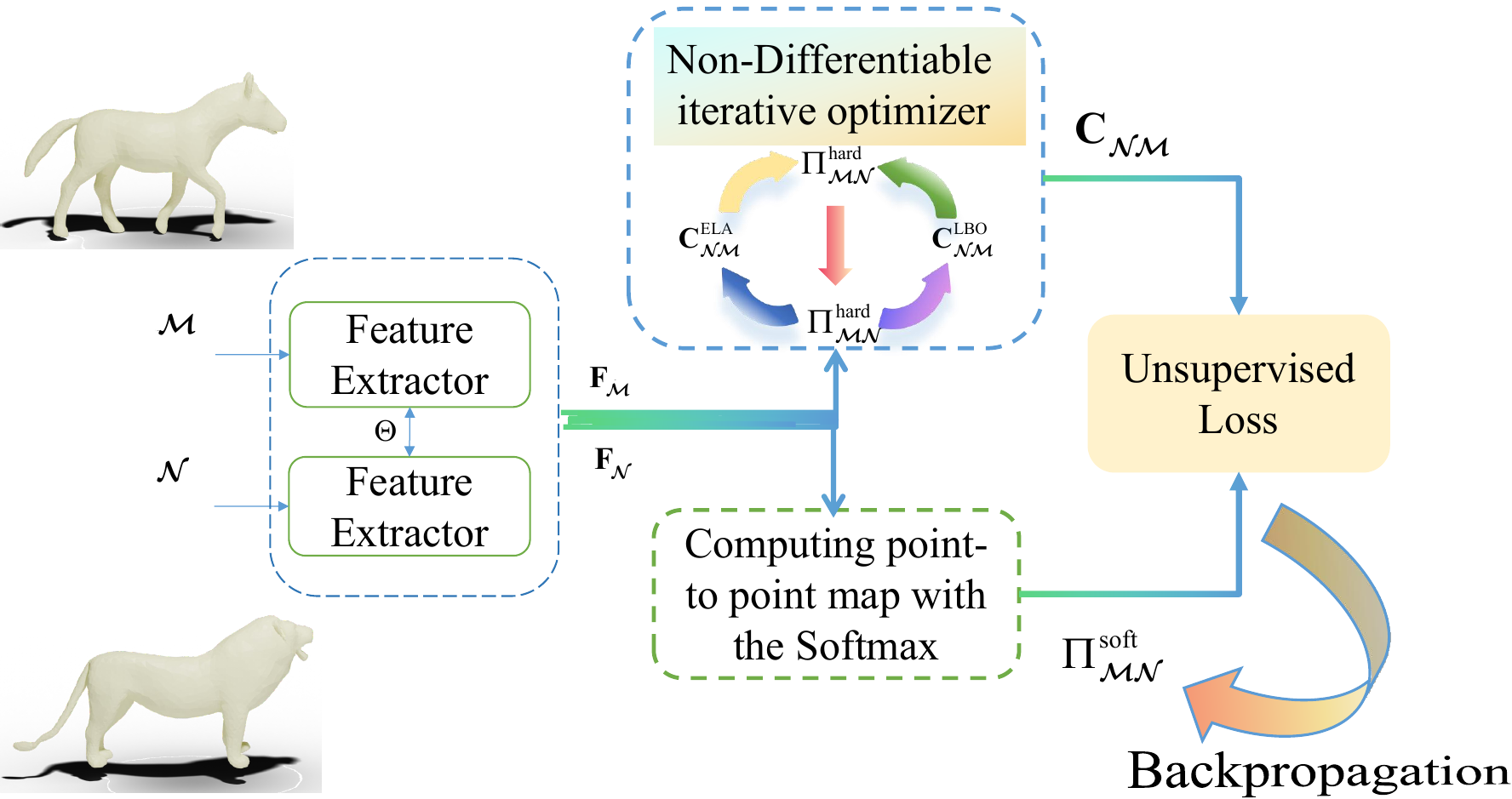}
	}
	\caption{A comparison of deep functional map architectures. (a) The conventional dual-branch framework which is fully differentiable and thus limited to differentiable solvers. (b) Our proposed MDND framework, which introduces a paradigm shift by incorporating a non-differentiable iterative refinement oracle. This oracle generates a high-quality target map, while gradients for training the feature extractor propagate exclusively through the parallel, differentiable branch. This design decouples refinement from learning and avoids potential gradient conflicts.}
	\label{fig:pipeline}
\end{figure*}

\subsection{Deep Functional Map} 
\label{DFM}
Deep functional map (DFM) methods have become state-of-the-art for non-rigid shape correspondence. The standard DFM pipeline, illustrated in Figure \ref{fig:pipeline}, consists of three main stages: feature extraction, differentiable map estimation, and unsupervised loss computation.

\textit{\textbf{Feature Extraction.}} Given two input shapes, $\mathcal{M}$ and  $\mathcal{N}$, a trainable Siamese network $\mathcal{F}_\theta$ is employed to extract pointwise feature descriptors, $\mathbf{F}_{\mathcal{M}}$ and $\mathbf{F}_{\mathcal{N}}$, respectively. Here, $\theta$ represents the learnable network parameters. DiffusionNet \cite{Nicholas2022} has emerged as the de facto standard backbone for this task, as it excels at learning robust features that are invariant to discretization and aware of orientation.

\textit{\textbf{Differentiable Map Estimation.}} The learned features are then passed to a differentiable solver, which computes a functional map, $\mathbf{C}$, by minimizing an energy function such as the one specified in Equation~\eqref{eq:C}. The differentiability of this step is crucial as it allows gradients to flow back to the feature extraction network during training.

\textit{\textbf{Unsupervised Losses.}}  To train the feature extractor $\mathcal{F}_\theta$ without ground-truth correspondences, several structural regularizers are imposed on the computed functional maps. Foundational losses include an orthogonality regularizer, which encourages the map to be area-preserving, and a bijectivity regularizer, which enforces cycle consistency (i.e., mapping from $\mathcal{M}$ to $\mathcal{N}$ and back should approximate the identity). These are formulated as:
\begin{equation}\label{eq:l_orth}
	\mathcal{L}_\mathrm{orth}=\left \| \mathbf{C}_\mathcal{NM}^\mathrm{T}\mathbf{C}_\mathcal{NM}-\mathbf{I} \right \|_\mathrm{F}^2 + \left \| \mathbf{C}_\mathcal{MN}^\mathrm{T}\mathbf{C}_\mathcal{MN}-\mathbf{I} \right \|_\mathrm{F}^2,	  
\end{equation}
\begin{equation}\label{eq:l_bij}
	\mathcal{L}_\mathrm{bij}=\left \| \mathbf{C}_\mathcal{NM}\mathbf{C}_\mathcal{MN}-\mathbf{I} \right \|_\mathrm{F}^2 + \left \| \mathbf{C}_\mathcal{MN}\mathbf{C}_\mathcal{NM}-\mathbf{I} \right \|_\mathrm{F}^2. 
\end{equation}

More recently, \cite{Cao2023Unsupervised} introduced a coupling loss used to enforce consistency between the map computed by the differentiable solver and the one derived from the soft pointwise map $\Pi_{\mathcal{MN}}^\mathrm{soft}$, which is derived from the feature similarity using a Softmax operation.:
\begin{equation}\label{label:l_couple}
	\mathcal{L}_\mathrm{couple}=\left \| \mathbf{C}_\mathcal{NM}-\Phi_{\mathcal{M}}^\dagger \Pi_\mathcal{MN}^\mathrm{soft}\Phi_\mathcal{N}   \right \|_\mathrm{F}^2 
\end{equation}
This encourages the learned functional map to correspond to a valid pointwise map, significantly improving matching accuracy.

Our MDND framework fundamentally departs from this standard pipeline. Instead of relying on a differentiable solver to compute an initial map, \textit{we leverage a powerful, non-differentiable iterative refinement algorithm}. This algorithm directly optimizes the functional map to a much higher quality, and its output is then used as a supervisory signal to guide the learning of the feature network. This key difference allows us to break free from the constraints of end-to-end differentiability and integrate the strengths of axiomatic refinement into the deep learning process.

\section{Method}
Traditional iterative optimization methods are crucial in functional map computations; however, their direct integration into deep learning frameworks has been hindered by their non-differentiable nature. To overcome this limitation, we introduce a novel paradigm called MDND (\textbf{M}erging \textbf{D}ifferentiable and \textbf{N}on-\textbf{D}ifferentiable Branches), which facilitates information exchange between differentiable and non-differentiable components within deep functional maps. This approach challenges the prevailing assumption that only differentiable operations are suitable for such frameworks. Figure \ref{fig:pipeline} offers a concise overview of our method. Specifically, the first branch generates hard correspondence derived from learned features and serves as the input to produce a well-structured functional map using a novel iterative optimization solver (see Algorithm \ref{power1}). This map subsequently serves as a supervisory signal to guide the backpropagation of the second Differentiable branch, which operates on soft correspondences. A detailed explanation of our method is provided in the following.

\subsection{Feature Extractor} \label{Feature}
The first core component of our network is the Deep Feature Module, implemented as a Siamese network with shared weights. This module extracts features from the source and target shapes, which consist of $m$ and $n$ vertices, respectively. We employ DiffusionNet \cite{Nicholas2022}, a state-of-the-art surface feature extractor that utilizes diffusion across the surface to generate features resilient to discretization variations. Furthermore, DiffusionNet incorporates a spatial gradient operation to effectively address bilateral symmetry. The extracted features for the source and target shapes are denoted by $\mathbf{F}_\mathcal{M} \in \mathbb{R}^{m \times d}$ and $\mathbf{F}_\mathcal{N} \in \mathbb{R}^{n \times d}$, respectively, where $d$ signifies the dimensionality of the learned features.

\subsection{Non-Differentiable Iterative Refinement}\label{subsec:hwf}
Inspired by the efficiency of spectral filtering techniques like MWP \cite{Ling2021}, we propose a novel and efficient iterative optimization method, which we term \textbf{Hybrid Wavelet Filtering (HWF)}. The core idea is that a high-quality correspondence can be recovered through a remarkably simple iterative loop: (1) converting a pointwise map to its functional map representation, (2) refining this functional map via spectral filtering, and (3) converting the refined map back to an updated pointwise map. Starting with an initial correspondence, iterating these steps rapidly converges to an accurate solution at a very low computational cost.

Our primary contribution, which distinguishes HWF from MWP, is the \textbf{generalization of this filtering process to hybrid spectral bases}. Instead of operating solely on the standard LBO basis, our method leverages a combined basis of both LBO and ELA eigenfunctions. This is crucial for improving robustness in challenging scenarios involving non-isometric deformations and topological noise, where the LBO basis alone is insufficient. While inspired by MWP, we present a completely different theoretical derivation, which is provided in detail in the appendix.

The complete algorithmic workflow is detailed in Algorithm~\ref{power1}. In the algorithm,
$\Phi^{\mathrm{LBO}}_{\mathcal{M}}$ and $\Phi^{\mathrm{ELA}}_{\mathcal{M}}$ are matrices composed of the LBO and ELA eigenfunctions, while $\Lambda^\mathrm{LBO}$ and $\Lambda^\mathrm{ELA}$ are the diagonal matrices of their corresponding eigenvalues. $\{g(s_l\lambda)\}_{l=1}^{L}$ represents a family of spectral manifold wavelet filters, $\Pi^{\mathrm{hard}}_\mathcal{MN}$ is  a vertex index sequence, and $(\mathbf{C}_{\mathcal{NM}}^\wedge)^\ast$ denotes to the adjoint operator of $\mathbf{C}_{\mathcal{NM}}^\wedge$.

In the following section, we will detail how we embed this powerful, non-differentiable HWF algorithm into our deep functional map framework to serve as a supervisory oracle, guiding the learning of robust feature descriptors.

\begin{algorithm}[!ht]
	\caption{HWF for Correspondence}
	\label{power1}
	\begin{algorithmic} 
		\STATE \textbf{Input}: Initialize pointwise map $\Pi^{\mathrm{hard}}_\mathcal{MN}$
		\STATE \textbf{Output}: Refined $\Pi^{\mathrm{hard}}_\mathcal{MN}$, $(\mathbf{C}^\mathrm{ELA}_{\mathcal{NM}})^\wedge$, $(\mathbf{C}^\mathrm{LBO}_{\mathcal{NM}})^\wedge$
		\STATE Iterative updates between $\mathbf{C^{\mathrm{LBO}}_{\mathcal{NM}}}$, $\mathbf{C^{\mathrm{ELA}}_{\mathcal{NM}}}$ and $\Pi^{\mathrm{hard}}_\mathcal{MN}$
		\STATE For $i=1$ to $maxIter$ do         \STATE\hspace{1cm}$\mathbf{C}^{\mathrm{ELA}}_{\mathcal{NM}}=(\Phi^{\mathrm{ELA}}_{\mathcal{M}})^\dagger\Pi_{\mathcal{MN}}^\mathrm{hard}\Phi^{\mathrm{ELA}}_{\mathcal{N}}$
		\vspace{0.1cm}        \STATE\hspace{1cm}$(\mathbf{C}^\mathrm{ELA}_{\mathcal{NM}})^\wedge=\sum_{l=1}^{L} g(s_l \boldsymbol{\Lambda}_\mathcal{M}^\mathrm{ELA}) \mathbf{C}^{\mathrm{ELA}}_{\mathcal{NM}} g(s_l \boldsymbol{\Lambda}_\mathcal{N}^\mathrm{ELA})$
		\vspace{0.1cm}         \STATE\hspace{1cm}$\mathbf{C}^{\mathrm{LBO}}_{\mathcal{NM}}=(\Phi^{\mathrm{LBO}}_{\mathcal{M}})^\dagger\Pi_{\mathcal{MN}}^\mathrm{hard}\Phi^{\mathrm{LBO}}_{\mathcal{N}}$   
		\vspace{0.1cm}          \STATE\hspace{1cm}$(\mathbf{C}^{\mathrm{LBO}}_{\mathcal{NM}})^\wedge=\sum_{l=1}^{L} g(s_l \boldsymbol{\Lambda}_\mathcal{M}^\mathrm{LBO}) \mathbf{C}^{\mathrm{LBO}}_{\mathcal{NM}} g(s_l \boldsymbol{\Lambda}_\mathcal{N}^\mathrm{LBO})$   
		\vspace{0.1cm} \STATE\hspace{1cm}$\Pi^{\mathrm{hard}}_\mathcal{MN}=\mathrm{NNsearch}
		\begin{pmatrix}
			\Phi_\mathcal{N}^{\mathrm{ELA}}((\mathbf{C}^{\mathrm{ELA}}_{\mathcal{NM}})^\wedge)^\ast& \Phi_\mathcal{M}^{\mathrm{ELA}}\\
			\Phi_\mathcal{N}^{\mathrm{LBO}}((\mathbf{C}^{\mathrm{LBO}}_{\mathcal{NM}})^\wedge)^\ast& \Phi_\mathcal{M}^{\mathrm{LBO}}
		\end{pmatrix}$     
		\STATE end
	\end{algorithmic}
\end{algorithm}

\subsection{The MDND Dual-Branch Architecture}
\label{FM-module}
Our framework is built on a dual-branch architecture designed to leverage the strengths of both differentiable learning and non-differentiable optimization. One branch operates in a fully differentiable manner to enable gradient-based training, while the other, non-differentiable branch acts as a powerful refinement oracle to provide high-quality supervision.

\textbf{\textit{Non-Differentiable Refinement Branch.}} The purpose of this branch is to generate a highly accurate target correspondence. It begins by computing an initial hard pointwise map, $\Pi^{\mathrm{hard}}_{\mathcal{MN}}$, via a simple nearest-neighbor search on the learned features $\mathbf{F}_\mathcal{M}$ and $\mathbf{F}_\mathcal{N}$:
\begin{equation}
	\Pi^{\mathrm{hard}}_{\mathcal{MN}}=\mathrm{NNsearch}({\mathbf{F}_\mathcal{N}},{\mathbf{F}_\mathcal{M}})
\end{equation}
This hard map, which can be memory-efficiently represented as a single vector of indices in implementation, serves as the input to our Hybrid Wavelet Filtering (HWF) algorithm (Algorithm~\ref{power1}). The HWF iteratively refines this initial guess, producing high-quality functional maps based on hybrid spectral bases, $(\mathbf{C}^\mathrm{ELA}_{\mathcal{NM}})^\wedge$ and $(\mathbf{C}^\mathrm{LBO}_{\mathcal{NM}})^\wedge$. Since the entire HWF process is parameter-free and non-differentiable, we do not track gradients through this branch, treating its output solely as a supervisory signal.

\textbf{\textit{Differentiable Learning Branch.}} This branch is responsible for learning the feature extractor $\mathcal{F}_\theta$. To maintain a differentiable path for backpropagation, we compute a soft pointwise map, $\Pi^{\mathrm{soft}}_{\mathcal{MN}}$, from the learned features. This is achieved by calculating a feature similarity matrix and applying a temperature-scaled Softmax operator:
\begin{equation}
	\Pi^{\mathrm{soft}}_{\mathcal{MN}}=\mathrm{Softmax}({\mathbf{F}_\mathcal{M}}{\mathbf{F}^\mathrm{T}_\mathcal{N}}/\tau),
\end{equation}
where $\tau$ is the temperature parameter that controls the softness of the resulting map. The output of this branch, $\Pi^{\mathrm{soft}}_{\mathcal{MN}}$, is a soft correspondence that can be directly used in our loss function to update the network weights. 

\begin{table*}[h!t]
	\centering
		\begin{tabular}{l *{8}{r}} 
			\toprule
			Method / Dataset  & \makecell[c]{F\_r/\\ F\_r} & \makecell[c]{S\_r/\\ S\_r} & \makecell[c]{F\_r/\\ S\_r} & \makecell[c]{S\_r/\\ F\_r} & SMAL & \makecell[c]{DT4D-H\\ inter} & \makecell[c]{DT4D-H\\ intra} & TOPKIDS\\ 
			\midrule
			BCICP & 6.1 & 11.0 & - & - & 28.6 & - &- &-\\
			ZoomOut & 6.1 & 7.5 & - & - & 38.4 & 29.0 &4.0&33.7 \\
			SmoothShells & 2.5 & 4.7 & - & - &36.1&6.4&1.2&10.8 \\
			DiscreteOp & 5.6 & 13.1 & - & - & 38.1&27.6&3.6 &35.5\\
			MWP & 3.1 & 4.1 & - & - & 20.9 &25.4&25.4&5.7\\
			\midrule
			FMNet & 11.0 & 17.0 & 30.0 & 33.0 & 42.0&38.0 &9.6&- \\
			GeomFmaps & 3.5 & 4.3 & 4.8 & 4.0 & 8.4&4.2&1.9&- \\
			\midrule
			DeepShells & 1.9 & 4.5 & 6.8 & 5.5 & 28.7&31.1&3.4 &13.7\\
			DUOFMNet & 2.5 & 2.6 & 4.2 & 2.7 & 6.7&15.8&2.6&- \\
			AttentiveFMaps & 1.9 & 2.1 &  \underline{2.6} & 1.9 & 4.4 &11.6&1.7&23.4\\
			ConsistentFMaps & 2.3 & 2.4 &  \underline{2.6} & 2.5 & 5.2  & 6.1&1.2&-\\
			RFMNet & 1.6 & 4.5 & 5.3 & 2.1  & 4.4&5.4&1.8&\underline{4.9}\\
			ULRSSM & \underline{1.6} & \textbf{1.8} & 6.4 & 4.5 & 4.4 &\underline{4.1}&\textbf{0.9}&9.2 \\
			MSSFMaps & 1.9 & 2.6 & 4.6 & 2.8 & 4.3 &\underline{4.1}&1.8 &-\\
			HybridFMaps & \textbf{1.5} & \textbf{1.8} & 8.2 & \underline{1.8} & \underline{3.3}&\textbf{3.5} &\underline{1.0} &5.0\\
			MSRFMNet & 1.7 & 2.1 &  \underline{2.6} & 2.0 & 4.5&36.2 &1.5&33.2\\
			DFAFM & \underline{1.6}& \underline{1.9} & 2.7 &1.9 & 3.9&4.2 &\textbf{0.9}&6.3\\
			Ours & \underline{1.6} & \underline{1.9} & \textbf{2.1} & \textbf{1.6} & \textbf{3.1}&4.4&\underline{1.0}&\textbf{3.5}\\
			\bottomrule
		\end{tabular}
		\caption{Quantitative comparison with state-of-the-art methods. We report the mean geodesic error ($\times 100$) across datasets representing near-isometric, non-isometric, and topologically noisy scenarios. Methods are grouped by category. The \textbf{best} and \underline{second-best} results are highlighted in bold and underlined, respectively.}
	\label{Result}
\end{table*}
\subsection{Unsupervised Loss Function}
\label{loss}
Traditional DFM frameworks often rely on multiple structural regularizers, such as orthogonality and bijectivity losses. Balancing the weights of these competing terms can be challenging and can complicate the optimization landscape. To avoid this, we adopt a single, streamlined unsupervised loss function inspired by recent work \cite{hu2023rfmnet}. Our loss directly enforces consistency between the output of our two branches: the soft pointwise map from the differentiable branch and the high-quality functional map from the non-differentiable oracle. The alignment loss constructed upon hybrid spectral bases is defined as:
\begin{equation}
	\mathcal{L}_\mathrm{align}=\left\|{\Phi}_\mathcal{M}-\Pi^{\mathrm{soft}}_{\mathcal{MN}}{\Phi}_\mathcal{N}(\mathbf{C}_{\mathcal{NM}}^\wedge)^\ast\right\|_\mathrm{F}^2.
\end{equation}
During backpropagation, $\mathbf{C}_{\mathcal{NM}}^\wedge$ obtained by the HWF refiner is treated as a fixed, detached constant, ensuring that gradients only flow through the differentiable branch to update the feature extractor. This elegant formulation allows the network to learn meaningful correspondences by chasing a high-quality, iteratively refined target, without the need for multiple, hand-weighted loss terms.

\section{Experiments}
In this section, we conduct extensive experiments to evaluate our method. We compare MDND against a diverse set of previous approaches across a broad range of challenging scenarios, from near-isometric matching to settings with significant non-isometric deformations and topological noise.

\subsection{Implementation Details}
Our method was implemented in PyTorch, and all experiments were run on a single NVIDIA RTX 4090 GPU. Following standard evaluation protocols, we report the mean geodesic error, normalized by the square root of the source shape's area. For fair comparison, we use 128-dimensional HKS descriptors \cite{Sun2009} as input features for all methods and datasets. We set the LBO basis size to 128 and the ELA basis size to 200 for all datasets, except for SMAL, where it was reduced to 100 due to the nature of the shapes. No post-processing or test-time adaptation was applied to our results. For readability, all reported geodesic errors in our tables are multiplied by 100.

\subsection{Comparison with State-of-the-Art}
\textbf{\textit{Baselines.}} We compare our method against a comprehensive set of recent and influential works, which can be categorized as follows:
\begin{itemize}
	\item \textit{Axiomatic methods:} BCICP \cite{Ren2018}, ZoomOut \cite{Melzi2019}, Smooth-Shells \cite{eisenberger2020smooth}, DiscreteOp \cite{Ren2021}, and MWP \cite{Ling2021}.
	\item \textit{Supervised methods:} FMNet~\cite{litany2017deep} and GeomFmap~\cite{Donati2020}.
	\item \textit{Unsupervised methods:} A wide range of recent approaches including DeepShells \cite{Eisenberger2020a}, DUOFMNet \cite{donati2022deep}, AttentiveFMaps \cite{Li2022LearningMF}, ConsistentFMaps \cite{sun2023spatially}, RFMNet \cite{hu2023rfmnet}, ULRSSM \cite{Cao2023Unsupervised}, MSSFMaps \cite{magnet2024memory}, HybridFMaps \cite{bastian2024hybrid}, MSRFMNet \cite{liu2024multiscale} and DFAFM \cite{luo2025deep}.
\end{itemize}

\begin{figure}[htbp]
	\begin{centering}
		\includegraphics[width=0.99\linewidth]{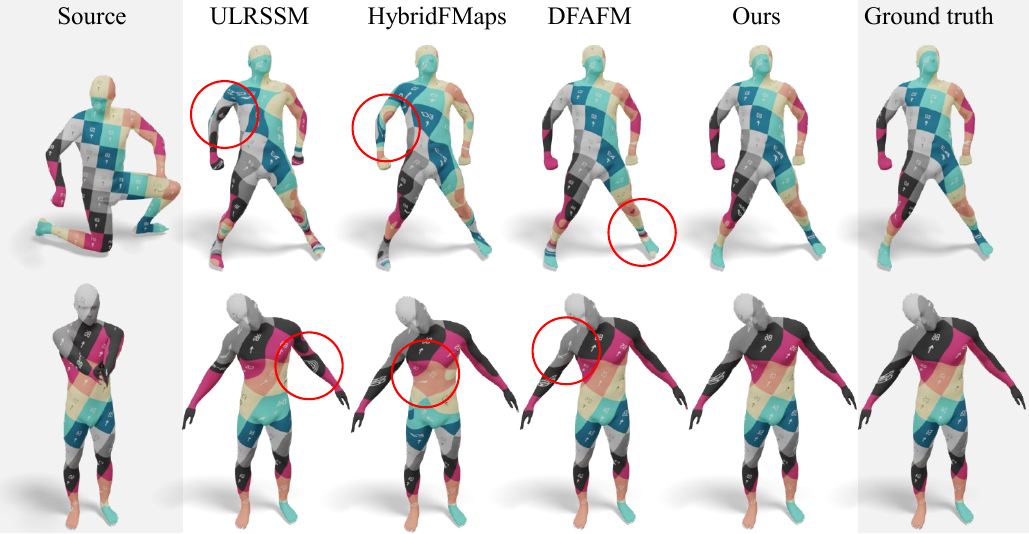}
	\end{centering}
	\caption{Cross-dataset generalization results. Correspondence quality is visualized using texture transfer when training and testing on different datasets.}
	\label{fig:f_s_generalization}
\end{figure}

\textbf{\textit{Near-isometric Shape Matching.}} We first evaluate MDND on two standard near-isometric benchmarks: FAUST \cite{bogo2014faust} and SCAPE \cite{Anguelov2005}. As shown in Table \ref{Result}, our method achieves performance comparable to the state-of-the-art on the standard remeshed FAUST and SCAPE test sets. Notably, our method excels in the generalization tests (F→S and S→F), indicating its ability to learn robust features that transfer well across different shape collections.

\textbf{\textit{Non-isometric Shape Matching.}} To assess performance on more challenging non-isometric shapes, we use the SMAL \cite{zuffi20173d} and DT4D-H \cite{magnet2022smooth} datasets. On SMAL, which contains various tetrapod species, Table \ref{Result} shows that our method significantly outperforms approaches that rely solely on the LBO basis, underscoring the importance of extrinsic information for non-isometric correspondence. On the DT4D-H dataset, our method demonstrates superiority over most competitors in both intra-class and inter-class matching scenarios. 

\textbf{\textit{Matching with Topological Noise.}} Finally, we test the robustness of our method on the TOPKIDS dataset \cite{lahner2016shrec}, which features near-isometric deformations corrupted by significant topological artifacts. The results in Table \ref{Result} are striking: our method achieves the best performance, improving upon the next-best approach by a remarkable 30\%. This highlights the exceptional robustness of our refinement oracle and learning framework. 

\subsection{Ablation Studies}
Our primary contributions are twofold: the integration of a non-differentiable iterative refinement method into the deep functional map framework, and the generalization of the MWP algorithm to hybrid bases (HWF). To validate the effectiveness of each component, we conduct the following targeted ablation studies.

\textbf{\textit{(1) Differentiable Solver vs. Non-Differentiable Iterative Refinement.}} 
To verify the benefit of embedding a non-differentiable optimizer within our learning framework, we compare our full MDND model against a baseline that adheres to a more traditional, fully differentiable pipeline. In this baseline, we replace our non-differentiable iterative refinement branch with a standard differentiable solver that computes the functional map directly from the features via Equation (2). We perform experiments on representative datasets covering near-isometric (SCAPE), non-isometric (SMAL), and topologically noisy (TOPKIDS) shapes.

The results, presented in Table \ref{n-d}, demonstrate that our proposed MDND framework with non-differentiable refinement consistently achieves superior matching accuracy across all three categories of datasets. This confirms that using a powerful, non-differentiable oracle to generate a high-quality supervisory signal is more effective than relying on a purely differentiable solver.
\begin{table}
	\begin{center}
		\begin{tabular}{crrrr}
			\toprule
			Method/Dataset   &{\makecell[c]{SCAPE}}& {\makecell[c]{SMAL}}&{\makecell[c]{TOPKIDS}} \\
			\midrule
			Differentiable Solver  & 2.0 &4.2  &9.4   \\
             Non-Differentiable \\
             Iterative Refinement & \textbf{1.9} & \textbf{3.1}&\textbf{3.5}    \\
		
			\bottomrule
		\end{tabular}
	\end{center}
        \caption{Impact of the non-differentiable refinement branch. To measure the effectiveness of our core contribution, we replaced our non-differentiable oracle with a standard differentiable solver. The results show a clear performance drop for the fully differentiable version, confirming that our architecture is crucial for achieving state-of-the-art accuracy.}
	\label{n-d}
\end{table}

\textbf{\textit{(2) Analysis of the Refinement Oracle: MWP vs. HWF.}}
To address the limitations of the LBO basis in non-isometric scenarios, we proposed HWF, which generalizes MWP to a hybrid basis. To isolate and validate the effectiveness of this contribution, we configure the MDND framework with three different refinement oracles: 
\begin{itemize}
	\item MWP (LBO only): The standard MWP algorithm using only the LBO basis.
	\item MWP (ELA only): MWP adapted to use only the ELA basis.
	\item HWF (LBO + ELA): Our proposed method using the hybrid basis.
\end{itemize}
All other experimental settings remain constant. The results, shown in Table \ref{mwp}, clearly indicate that our HWF (LBO + ELA) achieves the best matching performance across all datasets. This validates our hypothesis that generalizing the refinement to a hybrid basis provides a significant advantage, leading to a more robust oracle capable of handling diverse and challenging geometric settings.
\begin{table}
	\begin{center}
		\begin{tabular}{crrrr}
			\toprule
			Method/Dataset   &{\makecell[c]{SCAPE}}& {\makecell[c]{SMAL}}&{\makecell[c]{TOPKIDS}} \\
			\midrule
			MWP (LBO only) &2.2  & 4.9 &14.7   \\
            MWP (ELA only) & 2.3 & 4.1&5.5    \\
			 HWF (LBO + ELA) &\textbf{1.9}  & \textbf{3.1}&\textbf{3.5}      \\
			\bottomrule
		\end{tabular}
	\end{center}
        \caption{Effectiveness of the Hybrid-Basis Refiner (HWF). To isolate the contribution of our proposed HWF algorithm, we compare its performance against single-basis alternatives. The results confirm that combining both the LBO and ELA bases within our HWF refiner is crucial for achieving the best performance.}
	\label{mwp}
\end{table}

\section{Conclusions}
In this paper, we introduced MDND, a novel deep functional map framework that merges the power of deep learning with the robustness of traditional axiomatic optimization. Our core contribution was a dual-branch architecture that leverages a powerful, non-differentiable iterative refinement oracle to provide high-quality supervision for a feature-learning network. This was enabled by our new HWF algorithm, which operates on expressive hybrid (LBO+ELA) bases. Our approach simplifies the training process and achieves state-of-the-art accuracy, particularly on challenging non-isometric and topologically noisy shapes.

Despite these strong results, we acknowledge the inherent limitations of a purely spectral approach. A promising direction for future work is to integrate our framework with explicit spatial deformation models, potentially bridging the gap between the spectral and spatial domains for even greater robustness.

\section{Acknowledgments}
This work was supported by the Natural Science Foundation of China (No. 62172447, 62302530), the Hunan Provincial Natural Science Foundation of China (No. 2023JJ40769), and the Postgraduate Research and Innovation Project of Hunan Province (No. CX20250157). This work was supported in part by the High Performance Computing Center of Central South University.

\bibliography{aaai2026}
\clearpage
\appendix
\twocolumn[
\begin{center}
    {\Large \bf MDND: Unsupervised Learning Guided by Non-Differentiable Refinement for Shape Correspondence\\[0.5em] Supplementary Materials}
\end{center}
]
In these supplementary materials, we provide a detailed theoretical derivation of our proposed Hybrid Wavelet Filtering (HWF) algorithm and include additional experimental results that could not be accommodated in the main paper due to space constraints.

\section{Derivation of the HWF Algorithm}
\label{AUFM}
In this section, we provide the theoretical justification for our iterative refinement algorithm, HWF. We begin by reviewing the necessary preliminaries on spectral manifold wavelets and their relationship with functional maps. We then extend the existing commutativity constraint to a hybrid basis and formulate a corresponding optimization objective. Finally, we detail the iterative procedure for solving this objective.

\subsection{Background: Functional Maps and Spectral Wavelet Commutativity}
Wavelet analysis is a cornerstone of signal processing, valued for its multi-scale analysis capabilities. To generalize this concept to non-Euclidean domains, Hammond et al. \cite{Hammond2011} introduced spectral graph wavelets, which are defined in the spectral domain analogous to classical Fourier-based wavelet operations. This formulation preserves the desirable localization properties of wavelets while being computationally efficient. This idea was subsequently extended to Riemannian manifolds by replacing the graph Laplacian with the Laplace-Beltrami Operator (LBO) \cite{Li2013, Hu2019}.

\begin{Definition}[Spectral Manifold Wavelet Operator]
Let $g(\cdot)$ be a non-negative, continuous function acting as a kernel, such as a low-pass or band-pass filter. The spectral manifold wavelet operator at scale $s$ is defined as: 
\begin{equation}
    \mathcal{W}_s = g(s\Delta),
\end{equation}
where $\Delta$ denotes the Laplace–Beltrami operator (LBO). Given the eigendecomposition of the LBO, $\Delta=\Phi^{\mathrm{LBO}} \Lambda^{\mathrm{LBO}}(\Phi^{\mathrm{LBO}})^\dagger$, the matrix representation of this operator in the LBO eigenbasis is given by $\mathbf{W}_{s} = \Phi^{\mathrm{LBO}} g(s\boldsymbol\Lambda^{\mathrm{LBO}})(\Phi^{\mathrm{LBO}})^\dagger$.
\end{Definition}

A key property for shape correspondence is that, for near-isometric mappings, the functional map operator commutes with the spectral manifold wavelet operator. This has been established in prior work \cite{liu2024multiscale} and is formalized below.

\begin{Remark}[Commutativity Property]
\label{commutativity}
Let $\mathcal{M}$ and $\mathcal{N}$ be two Riemannian manifolds, and let $T: \mathcal{M} \to \mathcal{N}$ be an isometric mapping. Let  $T_F: \mathcal{L}^2 (\mathcal{N}) \to \mathcal{L}^2 (\mathcal{M})$ be the induced functional map operator and let $\mathcal{W}_{s}^\mathcal{M}$ and $\mathcal{W}_{s}^\mathcal{N}$ be the spectral manifold wavelet operators. Then, $T_F$ commutes with the wavelet operators:
\begin{equation}
    \label{cont_reg}
    T_{F}\mathcal{W}^{\mathcal{N}}_{s} = \mathcal{W}^{\mathcal{M}}_{s}T_{F}.
\end{equation}
In the discrete setting, using the LBO eigenbasis, this commutativity is expressed as a constraint on the functional map matrix $\mathbf{C^{\mathrm{LBO}}_{\mathcal{NM}}}$:
\begin{equation}\label{eq:discrete_commutativity}
  \mathbf{C^{\mathrm{LBO}}_{\mathcal{NM}}}g(s\boldsymbol{\Lambda}^{\mathrm{LBO}}_\mathcal{N} )=g(s\boldsymbol{\Lambda}^{\mathrm{LBO}}_\mathcal{M} )\mathbf{C^{\mathrm{LBO}}_{\mathcal{NM}}}.  
\end{equation}
\end{Remark}

\begin{proof}
    The proof follows directly from the definitions. The matrix representation of the operators are: $T_{F}=\Phi^{\mathrm{LBO}}_\mathcal{M}\mathbf{C^{\mathrm{LBO}}_\mathcal{NM}}(\Phi^{\mathrm{LBO}}_\mathcal{N})^\dagger$, $\mathbf{W}_{s}^\mathcal{M} = \Phi^{\mathrm{LBO}}_\mathcal{M} g(s\boldsymbol\Lambda^{\mathrm{LBO}}_\mathcal{M})(\Phi^{\mathrm{LBO}}_\mathcal{M})^\dagger$ and $\mathbf{W}_{s}^\mathcal{N} = \Phi^{\mathrm{LBO}}_\mathcal{N} g(s\boldsymbol\Lambda^{\mathrm{LBO}}_\mathcal{N})(\Phi^{\mathrm{LBO}}_\mathcal{N})^\dagger$. Substituting these into Equation~\eqref{cont_reg} yields:
    \begin{align*}      
        \Phi^{\mathrm{LBO}}_\mathcal{M}\mathbf{C^{\mathrm{LBO}}_\mathcal{NM}}(\Phi^{\mathrm{LBO}}_\mathcal{N})^\dagger \Phi^{\mathrm{LBO}}_\mathcal{N} g(s\boldsymbol\Lambda^{\mathrm{LBO}}_\mathcal{N})(\Phi^{\mathrm{LBO}}_\mathcal{N})^\dagger & = \\ \Phi^{\mathrm{LBO}}_\mathcal{M} g(s\boldsymbol\Lambda^{\mathrm{LBO}}_\mathcal{M})(\Phi^{\mathrm{LBO}}_\mathcal{M})^\dagger \Phi^{\mathrm{LBO}}_\mathcal{M}\mathbf{C^{\mathrm{LBO}}_\mathcal{NM}}(\Phi^{\mathrm{LBO}}_\mathcal{N})^\dagger
    \end{align*}
    Since $\Phi^\dagger\Phi=\mathbf{I}$, the above formula can be simplified to Equation~\eqref{eq:discrete_commutativity}.
\end{proof}

\subsection{Motivation for a Hybrid-Basis Approach}
The commutativity property in Remark~\ref{commutativity} provides a powerful constraint for near-isometric matching, as it is based on the intrinsic LBO basis. However, this very intrinsicality is a limitation. The LBO basis is, by definition, insensitive to extrinsic deformations, making it struggle to capture high-frequency details like sharp creases and bends. This inadequacy becomes a significant bottleneck in challenging non-isometric correspondence scenarios.

Inspired by recent work on the extrinsic elastic basis (ELA basis) \cite{hartwig2023elastic}, which excels at representing such fine-grained geometric features, we propose a method that integrates both intrinsic and extrinsic information. Our HWF algorithm is built upon a generalized commutativity constraint that operates on a hybrid basis composed of both LBO and ELA eigenfunctions.

In the subsequent sections, we will first formulate this hybrid-basis constraint as an energy function. Then, we will derive an efficient iterative algorithm to solve the corresponding optimization problem, demonstrating how to synergistically leverage both intrinsic and extrinsic spectral information for robust shape correspondence.

\subsection{Optimization Problem}
\label{Energy formulation}
As previously discussed, only using the LBO basis has limitations within the functional map framework. Therefore, we consider employing both the LBO basis and the ELA basis simultaneously. First, we present the hybrid basis function spaces $\Phi_\mathcal{M}\in \mathbb{R}^{m \times k}$ and $\Phi_\mathcal{N}\in \mathbb{R}^{n \times k}$ for manifolds $\mathcal{M}$ and $\mathcal{N}$ respectively, where $\Phi_\mathcal{M}=[\Phi^{\mathrm{LBO}}_{\mathcal{M}}  \Phi^{\mathrm{ELA}}_{\mathcal{M}}]$, $\Phi_\mathcal{N}=[\Phi^{\mathrm{LBO}}_{\mathcal{N}}  \Phi^{\mathrm{ELA}}_{\mathcal{N}}]$ and $k$  is the number of hybrid bases. From this, it can be observed that the hybrid bases matrix is formed by concatenating the two sets of bases column-wise. Since our optimization objective is based on Remark~\ref{commutativity}, we give the matrix form of the spectral manifold wavelet operator under the hybrid bases: $\mathbf{W}^{\mathcal{M}}_{s} = \Phi_\mathcal{M}g(s\boldsymbol\Lambda_\mathcal{M})\Phi_\mathcal{M}^\dagger $ and $\mathbf{W}^{\mathcal{N}}_{s} = \Phi_\mathcal{N}g(s\boldsymbol\Lambda_\mathcal{N})\Phi_\mathcal{N}^\dagger $. The aforementioned eigenvalue matrices are diagonal matrices composed of the eigenvalues derived from the LBO operator and those generated by the shell energy decomposition. Similarly, we present the matrix representation of the functional map $T_{F}$ under the hybrid bases as follows: $T_{F}=\Phi_\mathcal{M}\mathbf{C_\mathcal{NM}}\Phi_\mathcal{N}^\dagger$, $\Phi_\mathcal{M}$ and $\Phi_\mathcal{M}$ are hybrid bases as mentioned before, $\mathbf{C_\mathcal{NM}}\in \mathbb{R}^{k\times k}$ is the matrix form of the functional map under hybrid bases. Accordingly, the  operator commutativity constraints may be recast under a hybrid bases framework as: 
\begin{equation}
\mathbf{C_{\mathcal{NM}}}g(s\boldsymbol{\Lambda}_\mathcal{N} )=g(s\boldsymbol{\Lambda}_\mathcal{M} )\mathbf{C_{\mathcal{NM}}}.
\label{zong_op}
\end{equation}
\begin{proof}
	First, the functional map matrix is block-partitioned in accordance with the basis function structure, obtaining $T_{F}=[\Phi^{\mathrm{LBO}}_{\mathcal{M}}  \Phi^{\mathrm{ELA}}_{\mathcal{M}}]\begin{bmatrix}
		\mathbf{C}^\mathrm{LBO}_\mathcal{NM}& \mathbf{C}^\mathrm{Hybrid1}_\mathcal{NM}\\
		\mathbf{C}^\mathrm{Hybrid2}_\mathcal{NM} &\mathbf{C}^\mathrm{ELA}_\mathcal{NM}
	\end{bmatrix}\begin{bmatrix}
		(\Phi^{\mathrm{LBO}}_{\mathcal{N}})^\dagger\\
		(\Phi^{\mathrm{ELA}}_{\mathcal{N}})^\dagger
	\end{bmatrix}$. $\mathbf{W}^{\mathcal{M}}_{s}=[\Phi^{\mathrm{LBO}}_{\mathcal{M}}  \Phi^{\mathrm{ELA}}_{\mathcal{M}}]\begin{bmatrix}
		g(s\boldsymbol\Lambda^\mathrm{LBO}_\mathcal{M})& \\
		&{}(s\boldsymbol\Lambda^\mathrm{ELA}_\mathcal{M})
	\end{bmatrix}\begin{bmatrix}
		(\Phi^{\mathrm{LBO}}_{\mathcal{M}})^\dagger\\
		(\Phi^{\mathrm{ELA}}_{\mathcal{M}})^\dagger
	\end{bmatrix}$ and $\mathbf{W}^{\mathcal{N}}_{s}=[\Phi^{\mathrm{LBO}}_{\mathcal{N}}  \Phi^{\mathrm{ELA}}_{\mathcal{N}}]\begin{bmatrix}
		g(s\boldsymbol\Lambda^\mathrm{LBO}_\mathcal{N})& \\
		&g(s\boldsymbol\Lambda^\mathrm{ELA}_\mathcal{N})
	\end{bmatrix}\begin{bmatrix}
		(\Phi^{\mathrm{LBO}}_{\mathcal{N}})^\dagger\\
		(\Phi^{\mathrm{ELA}}_{\mathcal{N}})^\dagger
	\end{bmatrix}$. Consequently, the matrix representation of operator commutation under the hybrid bases is given by: $\begin{bmatrix}
		\mathbf{C}^\mathrm{LBO}_\mathcal{NM} & \mathbf{C}^\mathrm{Hybrid1}_\mathcal{NM}\\
		\mathbf{C}^\mathrm{Hybrid2}_\mathcal{NM} &\mathbf{C}^\mathrm{ELA}_\mathcal{NM}
	\end{bmatrix}\begin{bmatrix}
		g(s\boldsymbol\Lambda^\mathrm{LBO}_\mathcal{N})& \\
		&g(s\boldsymbol\Lambda^\mathrm{ELA}_\mathcal{N})
	\end{bmatrix}=\begin{bmatrix}
		g(s\boldsymbol\Lambda^\mathrm{LBO}_\mathcal{M})& \\
		&g(s\boldsymbol\Lambda^\mathrm{ELA}_\mathcal{M})
	\end{bmatrix}\begin{bmatrix}
		\mathbf{C}^\mathrm{LBO}_\mathcal{NM}& \mathbf{C}^\mathrm{Hybrid1}_\mathcal{NM}\\
		\mathbf{C}^\mathrm{Hybrid2}_\mathcal{NM} &\mathbf{C}^\mathrm{ELA}_\mathcal{NM}
	\end{bmatrix}$. Due to the inconsistent spectral characteristics exhibited by the LBO operator and shell energy, we conclude that the functional map matrix blocks $\mathbf{C}^\mathrm{Hybrid1}_\mathcal{NM}$  and $\mathbf{C}^\mathrm{Hybrid2}_\mathcal{NM}$ admit only the trivial solution of zero matrices. Therefore, the final optimization objective under hybrid bases becomes: $\mathbf{C}_{\mathcal{NM}}g(s\boldsymbol{\Lambda}_\mathcal{N} )=g(s\boldsymbol{\Lambda}_\mathcal{M} )\mathbf{C}_{\mathcal{NM}}$, where $\mathbf{C}_\mathcal{NM}=\begin{bmatrix}
		\mathbf{C}^\mathrm{LBO}_\mathcal{NM}&\\
		&\mathbf{C}^\mathrm{ELA}_\mathcal{NM}
	\end{bmatrix}$, $g(s\boldsymbol\Lambda_\mathcal{N})=\begin{bmatrix}
		g(s\boldsymbol\Lambda^\mathrm{LBO}_\mathcal{N})& \\
		&g(s\boldsymbol\Lambda^\mathrm{ELA}_\mathcal{N})
	\end{bmatrix}$  and $g(s\boldsymbol\Lambda_\mathcal{M})=\begin{bmatrix}
		g(s\boldsymbol\Lambda^\mathrm{LBO}_\mathcal{M})& \\
		&(s\boldsymbol\Lambda^\mathrm{ELA}_\mathcal{M})
	\end{bmatrix}.$
\end{proof}
In Hilbert space, the Hilbert–Schmidt norm effectively captures geometric structural information across the domain and range of the mapping operator, including anisotropic metrics, outperforming the Frobenius norm in this regard. Furthermore, the utilization of the Hilbert–Schmidt norm plays a crucial role in the success of work~\cite{hartwig2023elastic}. Thus, given $L$ discrete scales $\left \{ s_{l} \right \} _{l=1}^{L} $, we reformulate the discrete version of Equation \eqref{zong_op} as an optimization objective using the Hilbert–Schmidt norm to measure the magnitude of linear operators:
\begin{equation}
	\begin{aligned}
		&\min_{\mathbf{C_{\mathcal{NM}}}}E(\mathbf{C_{\mathcal{NM}}}),\\
		&E(\mathbf{C_{\mathcal{NM}}}) = \sum_{l=1}^{L}\left\|\mathbf{C_{\mathcal{NM}}}g(s_l\boldsymbol\Lambda_\mathcal{N} )-g(s_l\boldsymbol\Lambda_\mathcal{M} ) \mathbf{C_{\mathcal{NM}}}\right\|_\mathrm{HS}^{2}\\
	\end{aligned}
	\label{T0}
\end{equation}

However, on the one hand, directly solving the optimization problem tends to yield a trivial solution. On the other hand, although a pointwise map can naturally induce a functional map, the converse does not always hold. To address this issue, we constrain the functional map derived from a pointwise map, as specified in Equation~\eqref{C}. Consequently, the functional map optimization problem can be equivalently transformed into the task of recovering the corresponding pointwise map, which can be reformulated as follows:
\begin{equation}
	\label{T}
	\min_{\Pi_\mathcal{MN}}\sum_{l=1}^{L}\left\|\mathbf{C_{\mathcal{NM}}}g(s_l\boldsymbol\Lambda_\mathcal{N} )-g(s_l\boldsymbol\Lambda_\mathcal{M} ) \mathbf{C_{\mathcal{NM}}}\right\|_\mathrm{HS}^{2}
\end{equation}
\begin{equation}
	\label{C}
	\text{where}\, \mathbf{C}_{\mathcal{NM}}=\Phi^\dagger_\mathcal{M}\Pi_{\mathcal{MN}}\Phi_\mathcal{N}
\end{equation}

To recover the pointwise map from Equation~\eqref{T}, a natural approach is to substitute the functional map using its pointwise representation as given in Equation~\eqref{C}. Unfortunately, it will lead to a trivial solution. So we adopt the half-quadratic splitting strategy, replacing only one of the functional maps instead of both, as done in~\cite{Melzi2019}. This results in two decoupled subproblems that can be solved more effectively.
\begin{equation}
	\label{T1}
	\min_{\Pi_{\mathcal{MN}}}\sum_{l=1}^{L}\left\|\Phi^\dagger_{\mathcal{M}}\Pi_{\mathcal{MN}}\Phi_{\mathcal{N}}\ g(s_l\boldsymbol\Lambda_\mathcal{N} )-g(s_l\boldsymbol\Lambda_\mathcal{M} ) \mathbf{C_{\mathcal{NM}}}\right\|_\mathrm{HS}^{2}
\end{equation}
\begin{equation}
	\label{C1}
	\min_{\mathbf{C}_{\mathcal{NM}}}\left\|\mathbf{C}_{\mathcal{NM}}-\Phi^\dagger_{\mathcal{M}}\Pi_{\mathcal{MN}}\Phi_{\mathcal{N}}\right\|_\mathrm{HS}^{2}
\end{equation}
The solution to Equation~\eqref{C1} is straightforward,  $\mathbf{C}_{\mathcal{NM}}=\Phi^\dagger_{\mathcal{M}}\Pi_{\mathcal{MN}}\Phi_{\mathcal{N}}$. In other words, given the pointwise map, the corresponding functional map can be derived. Obtaining the pointwise map from Equation~\eqref{T1} needs more rigorous logical formulation, which is elaborated carefully in the next section.

\subsection{Solution of the Pointwise Map}
\label{Solution of the pointwise map}
Owing to the multi-scale analysis of wavelets, Equation~\eqref{T1} has stronger constraints on the pointwise map, and we can get the following equation at each scale $s_{l}$.
\begin{equation}
	\label{m}
	\Phi^\dagger_{\mathcal{M}}\Pi_{\mathcal{MN}}\Phi_{\mathcal{N}}\ g(s_l\boldsymbol\Lambda_\mathcal{N} )=g(s_l\boldsymbol\Lambda_\mathcal{M} ) \mathbf{C_{\mathcal{NM}}}
\end{equation}
Using the tight wavelet frame~\cite{Ling2021}, which satisfies $\sum_{l=1}^Lg(s_l\lambda)^2 \equiv 1$, where $\lambda$ represents the eigenvalue of the LBO or ELA, we can reformulate the optimization problem for the pointwise map as follows: 
\begin{equation}
	\label{TT2}
	\min_{\Pi_{\mathcal{MN}}} \left\| 
	\Phi^{\dagger}_{\mathcal{M}}\Pi_{\mathcal{MN}} \Phi_{\mathcal{N}} 
	- \sum_{l=1}^{L} g(s_l \boldsymbol\Lambda_{\mathcal{M}}) \mathbf{C}_{\mathcal{NM}} g(s_l \boldsymbol\Lambda_{\mathcal{N}})
	\right\|_\mathrm{HS}^2.
\end{equation}

\begin{proof}
	Given the multi-scale superimposed constraints based on Equation~\eqref{m}:  $\sum_{l=1}^{L} \Phi^\dagger_{\mathcal{M}}\Pi_{\mathcal{MN}}\Phi_{\mathcal{N}} g(s_l\boldsymbol\Lambda_\mathcal{N} )=\sum_{l=1}^{L} g(s_l\boldsymbol\Lambda_\mathcal{M} ) \mathbf{C_{\mathcal{NM}}}$, then multiply $g(s_l\boldsymbol\Lambda_\mathcal{N} )$ on both sides, $\sum_{l=1}^{L} \Phi^\dagger_{\mathcal{M}}\Pi_{\mathcal{MN}}\Phi_{\mathcal{N}}g^{2}(s_l\boldsymbol\Lambda_\mathcal{N} )=\sum_{l=1}^{L} g(s_l\boldsymbol\Lambda_\mathcal{M} ) \mathbf{C_{\mathcal{NM}}}g(s_l\boldsymbol\Lambda_\mathcal{N} )$. Using $\sum_{l=1}^Lg(s_l\lambda)^2 \equiv 1$, the constraint on the pointwise map satisfies the following equation: 
	$\Phi^{\dagger}_{\mathcal{M}}\Pi_{\mathcal{MN}} \Phi_{\mathcal{N}} 
	= \sum_{l=1}^{L} g(s_l \boldsymbol\Lambda_{\mathcal{M}}) \mathbf{C}_{\mathcal{NM}} g(s_l \boldsymbol\Lambda_{\mathcal{N}})$.
\end{proof}

For notational simplicity, we define $\mathbf{C^{\wedge}_{\mathcal{NM}}}=\sum_{l=1}^{L} g(s_l \boldsymbol{\Lambda}_\mathcal{M})
\mathbf{C_{\mathcal{NM}}} g(s_l \boldsymbol{\Lambda}_\mathcal{N})$ . Notably, Equation~\eqref{TT2} only constrains the image of the pointwise map $\Pi_{\mathcal{MN}}$ within the basis space $\Phi_{\mathcal{M}}$. However, in a complete space, this constraint alone is insufficient. To address this, we introduce an additional regularization term to restrain the pointwise map based on the foundational lemmas proposed in ~\cite{hartwig2023elastic, bastian2024hybrid}. Before proceeding, we first present the following foundational lemmas~\cite{hartwig2023elastic, bastian2024hybrid}.

\begin{lemma}
	\label{HS}
	Let $\mathbf{F}\in\mathbb{R}^{n,m}$ with $n,m>0$  be a linear operator between two finite-dimensional Hilbert spaces, and the corresponding Hilbert–Schmidt norm, then
	:\\
	(a) for all injective $\Phi_{k}\in\mathbb{R}^{n,k},k>0$
	\begin{equation}
		\left\| 
		\mathbf{F}
		\right\|_\mathrm{HS}^2= \left\| 
		\Phi_{k}\Phi_{k}^{\dagger}\mathbf{F}
		\right\|_\mathrm{HS}^2+\left\| 
		(\mathbf{I}-\Phi_{k}\Phi_{k}^{\dagger})\mathbf{F}
		\right\|_\mathrm{HS}^2
	\end{equation}
	
	(b) and for all injective $\Phi_k\in\mathbb{R}^{m,k},k>0$
	\begin{equation}
		\left\| 
		\mathbf{F}
		\right\|_\mathrm{HS}^2= \left\| 
		\Phi_{k}\Phi_{k}^{\dagger}\mathbf{F}
		\right\|_\mathrm{HS}^2+\left\| 
		(\mathbf{I}-\Phi_{k}\Phi_{k}^{\dagger})\mathbf{F}
		\right\|_\mathrm{HS}^2
	\end{equation}
\end{lemma}
Leveraging the above lemma, we convert the constraint problem on a pointwise map into the following form:
\begin{equation}
	\label{T3}
	\min_{\Pi_{\mathcal{MN}}} \left\| 
	\Pi_{\mathcal{MN}} \Phi_{\mathcal{N}} 
	-\Phi_{\mathcal{M}}\mathbf{C}^{\wedge}_{\mathcal{NM}} 
	\right\|_\mathrm{HS}^2
\end{equation}

\begin{proof}
	It is clear that 
	$\min_{\Pi_{\mathcal{MN}}}\left\| \Phi^{\dagger}_{\mathcal{M}}
	\Pi_{\mathcal{MN}} \Phi_{\mathcal{N}} 
	- \mathbf{C^{\wedge}_{\mathcal{NM}}}
	\right\|_\mathrm{HS}^2$ is  equivalent to $\min_{\Pi_{\mathcal{MN}}}\left\| \Phi_{\mathcal{M}}\Phi^{\dagger}_{\mathcal{M}}
	\Pi_{\mathcal{MN}} \Phi_{\mathcal{N}} 
	- \Phi_{\mathcal{M}}\mathbf{C^{\wedge}_{\mathcal{NM}}}
	\right\|_\mathrm{HS}^2$. Using the $\Phi^{\dagger}_{\mathcal{M}}\Phi_{\mathcal{M}}=\mathbf{I}$, a new optimization objective is obtained by incorporating another regularization term into the original one, \textit{i.e.}, $\left\| \Phi_{\mathcal{M}}\Phi^{\dagger}_{\mathcal{M}}
	(\Pi_{\mathcal{MN}} \Phi_{\mathcal{N}} 
	- \Phi_{\mathcal{M}}\mathbf{C^{\wedge}_{\mathcal{NM}}})
	\right\|_\mathrm{HS}^2 +\left\| (\mathbf{I}-\Phi_{\mathcal{M}}\Phi^{\dagger}_{\mathcal{M}})
	(\Pi_{\mathcal{MN}} \Phi_{\mathcal{N}} 
	- \Phi_{\mathcal{M}}\mathbf{C^{\wedge}_{\mathcal{NM}}})
	\right\|_\mathrm{HS}^2=\left\|
	\Pi_{\mathcal{MN}} \Phi_{\mathcal{N}} 
	- \Phi_{\mathcal{M}}\mathbf{C^{\wedge}_{\mathcal{NM}}}
	\right\|_\mathrm{HS}^2$. Therefore, Equation~\eqref{T3} yields.
\end{proof}

We additionally introduce the property of the shape difference operator~\cite{hartwig2023elastic} to further strengthen the constraint in Equation~\eqref {T3}. The shape difference operator is defined as the product of the functional map and its adjoint operator ($\mathbf{C}_{\mathcal{NM}}\mathbf{C}^\ast _{\mathcal{NM}}$) and the deviation of the shape difference operator from the identity reflects disparities in area distortion, where $\mathbf{C}^\ast _{\mathcal{NM}}$ is the adjoint operator of the functional map $\mathbf{C}_{\mathcal{NM}}$. The definition of the adjoint operator relies on the following lemma\ref{addd} and lemma\ref{ad}. 

\begin{lemma}
	\label{addd}
	For functions $f$,$g$ in $\langle\Phi_k\rangle$, where $x,y\in\mathbb{R}^k$ are the corresponding basis coefficients, namely $f=\Phi_{k}x$ and $g=\Phi_{k}y$, we obtain $\langle f,g\rangle_\mathbf{M}=x^\mathrm{T}\mathbf{M}_ky$. Here, $\mathbf{M}_k=\Phi_k^\mathrm{T}\mathbf{M}\Phi_k\in\mathbb{R}^{k,k}$ is the mass matrix with respect to the reduced basis $\Phi_{k}$.
	where $\mathbf{M}$ is the area matrix on the manifold, and $\Phi_{k}$ denotes the truncated basis matrix.
\end{lemma}

\begin{lemma}
	\label{ad}
	Given the functional map $\mathbf{C}_{\mathcal{NM}}$ . The corresponding adjoint operator can be represented by: $\mathbf{C}^\ast _{\mathcal{NM}}=\mathbf{A}^{-1}_{\mathcal{M}}\mathbf{C}^\mathrm{T}_{\mathcal{NM}}\mathbf{A}_{\mathcal{N}}$, where $\mathbf{A}_{\mathcal{N}}$ and $\mathbf{A}_{\mathcal{M}}$ are the mass matrices with respect to the reduced basis $\Phi_{k}$ based on the shape $\mathcal{M}$ and $\mathcal{N}$ respectively.
\end{lemma}

\begin{proof}
	
	After obtaining the mass matrix under the reduced basis, we can derive the adjoint operator of the functional map according to the method in~\cite{hartwig2023elastic}.
	The following equations are defined in~\cite{hartwig2023elastic}.
	\begin{equation}
		\langle \mathbf{C}_{\mathcal{NM}}x,y\rangle_{\mathbf{A}_{\mathcal{N}}}=\langle x,\mathbf{C}_{\mathcal{NM}}^*y\rangle_{\mathbf{A}_{\mathcal{M}}}
	\end{equation}
	
	Where $f$ and $g$ are descriptor functions defined on manifolds $\mathcal{M}$ and $\mathcal{N}$, and $x$ and $y$ are the corresponding basis coefficients under the respective basis function spaces.
	$\langle \mathbf{C}_{\mathcal{NM}}x,y\rangle_{\mathbf{A}_{\mathcal{N}}}=(\mathbf{C}_{\mathcal{NM}}x)^\mathrm{T}\mathbf{A}_{\mathcal{N}}y=x^\mathrm{T}\mathbf{C}_{\mathcal{NM}}^\mathrm{T}\mathbf{A}_{\mathcal{N}}y$\\
	$\langle x,\mathbf{C}_{\mathcal{NM}}^*y\rangle_{\mathbf{A}_{\mathcal{M}}}=x^\mathrm{T}\mathbf{A}_{\mathcal{M}}\mathbf{C}_{\mathcal{NM}}^*y$\\
	$\mathbf{C}_{\mathcal{NM}}^*=\mathbf{A}_{M}^{-1}\mathbf{C}_{\mathcal{NM}}^\mathrm{T}\mathbf{A}_{N}$.
\end{proof}

The functional map in Equation~\eqref{T3} can be regarded as a filtered functional map for eigenvalues.  To fully leverage the advantages of the aforementioned shape difference operator,
we reformulate the optimization problem as shown in Equation~\eqref{TT4} to be consistent with Equation~\eqref{T3}.
\begin{equation}
	\label{TT4}
	\min_{\Pi_{\mathcal{MN}}} \left\| 
	\Pi_{\mathcal{MN}} \Phi_{\mathcal{N}} (\mathbf{C}^{\wedge}_{\mathcal{NM}})^{\ast}
	-\Phi_{\mathcal{M}}
	\right\|_\mathrm{HS}^2
\end{equation}

\begin{proof}
	\begin{equation*}
		\begin{split}
			& \left\| \Pi_{\mathcal{MN}} \Phi_{\mathcal{N}} - \Phi_{\mathcal{M}} (\mathbf{C^{\wedge}_{\mathcal{NM}}})\right\|_{\mathrm{HS}}^{2} \\
			= &\operatorname{tr}\left(\Phi^{\ast}_{\mathcal{N}}\Pi^{\ast}_{\mathcal{MN}}\Pi_{\mathcal{MN}}\Phi_{\mathcal{N}} - \Phi^{\ast}_{\mathcal{N}}\Pi^{\ast}_{\mathcal{MN}}\Phi_{\mathcal{M}}\mathbf{C^{\wedge}_{\mathcal{NM}}} \right. \\
			& \left. - (\mathbf{C^{\wedge}_{\mathcal{NM}}})^{\ast}\Phi^{\ast}_{\mathcal{M}}\Pi_{\mathcal{MN}}\Phi_{\mathcal{N}} + (\mathbf{C^{\wedge}_{\mathcal{NM}}})^{\ast}\Phi^{\ast}_{\mathcal{M}}\Phi_{\mathcal{M}}\mathbf{C^{\wedge}_{\mathcal{NM}}}\right)
		\end{split}
	\end{equation*}
	By applying $\mathrm{tr}(\mathbf{A}+\mathbf{B})=\mathrm{tr}(\mathbf{A})+\mathrm{tr}(\mathbf{B})$, $\mathrm{tr}(\mathbf{AB})=\mathrm{tr}(\mathbf{BA})$ and $(\mathbf{C}^{\wedge}_{\mathcal{NM}})^{\ast}\mathbf{C}^{\wedge}_{\mathcal{NM}}=\mathbf{I}$, 
	it is therefore clear that Equation\eqref{T3} is equal to Equation\eqref{TT4}.
\end{proof}
This final form is advantageous because it can be efficiently solved for the hard correspondence matrix $\Pi_{\mathcal{MN}}$ using a simple nearest-neighbor search.

\subsection{The Complete HWF Algorithm}
Based on the analysis above, our HWF algorithm is an iterative process that alternates between solving the two decoupled subproblems:
\begin{itemize}
    \item \textbf{Update Functional Map: } Given the current pointwise map $\Pi_{\mathcal{MN}}^{(i)}$, compute the updated functional map $\mathbf{C}_{\mathcal{NM}}^{(i)}$ using its projection onto the hybrid basis (solving Eq.~\eqref{C1}). That is, the functional map matrix is obtained as the product of the basis function matrix defined on manifold $\mathcal{M}$  and the basis function matrix defined on manifold $\mathcal{N}$, where the latter is constructed by performing row-index search based on a hard pointwise map.

    \item \textbf{Update Pointwise Map: } Given the updated functional map $\mathbf{C}_{\mathcal{NM}}^{(i)}$, compute the filtered map $\mathbf{C}_{\mathcal{NM}}^{\wedge}$ and sove for the new pointwise map $\Pi_{\mathcal{MN}}^{(i+1)}$via a nearest-neighbor search on Equation~\eqref{TT4}. 
For manifold $\mathcal{M}$, after concatenating the two sets of basis functions column-wise, each row can be regarded as a feature of a vertex. Similarly, for the basis function space of the manifold  $\mathcal{N}$ constructed by the functional map, one can perform column-wise concatenation and then apply nearest-neighbor search on each row to obtain the pointwise map.

\end{itemize}
This process, summarized in Algorithm \ref{power1}, iterates until convergence, synergistically leveraging both intrinsic and extrinsic information to find a robust correspondence.

\section{Additional Results}

\subsection{Matching on Anisotropically Remeshed Shapes}
To evaluate the robustness of our method to different mesh discretizations, we conduct tests on anisotropically remeshed versions of the FAUST (F\_a) and SCAPE (S\_a) datasets. For these tests, we use the models trained on their isotropically remeshed counterparts (F\_r and S\_r, respectively). The results, summarized in Table \ref{F_A}, show that our method achieves competitive performance, securing either the best or second-best results in most generalization scenarios and demonstrating strong resilience to changes in mesh structure.
\begin{table}[h!t]
	\centering
	\scalebox{0.99}{
		\begin{tabular}{l *{4}{r}} 
			\toprule
			Method / Dataset  & \makecell[c]{F\_r/\\ F\_a} &  \makecell[c]{F\_r/\\ S\_a}  & \makecell[c]{S\_r/\\ F\_a}& \makecell[c]{S\_r/\\ S\_a} \\ 
			\midrule
			BCICP &14.0 &8.5 &14.0 &8.5\\
			ZoomOut &15.0 &8.7 &15.0 &8.7 \\
			SmoothShells &5.0 &5.4 &5.0 &5.4\\
			DiscreteOp &14.6 &6.2 &14.6 &6.2\\
			MWP &8.7 &8.2 &8.7&8.2\\
			\midrule
			DeepShells & 12.0&16.0&15.0&10.0\\
			DUOFMNet & 3.0&4.4&3.1& 2.7\\
			AttentiveFMaps & 2.4&2.8&\underline{2.5}&2.3\\
			RFMNet & 3.6&\underline{2.6}&3.6&3.9\\
			ULRSSM & 2.5&8.9 &7.0&\underline{1.9}\\
			HybridFMaps &\textbf{2.0}&8.8&2.6&\textbf{1.8}\\
			DFAFM & \textbf{2.0}&2.9 &2.6&\underline{1.9}\\
			Ours & \underline{2.1} & \textbf{2.0}&\textbf{2.2}&\underline{1.9}\\
			\bottomrule
		\end{tabular}
	}
    \caption{Comparative results on anisotropically remeshed shapes. This table evaluates generalization performance when models trained on isotropically remeshed data are tested on anisotropically remeshed versions of FAUST (F\_a) and SCAPE (S\_a).}
	\label{F_A}
\end{table}

\subsection{Parameter Analysis}
In this section, we analyze the sensitivity of our method to two key hyperparameters: the number of elastic basis functions used in our hybrid refiner and the number of iterations performed by the HWF algorithm.

\textbf{\textit{Number of Elastic Bases.}} Our method leverages both intrinsic (LBO) and extrinsic (ELA) bases. While we follow related work \cite{hu2023rfmnet} for the number of LBO bases, the optimal number of ELA bases requires investigation. We experimented on the SCAPE dataset by varying the number of ELA basis functions. As shown in Figure \ref{num_scape}, the results indicate that using 200 ELA basis functions provides the best trade-off. Fewer bases lead to degraded performance, while a larger number offers diminishing returns at the cost of increased computational overhead.
\begin{figure}[h!t]
	\begin{centering}
		\includegraphics[width=0.96\linewidth]{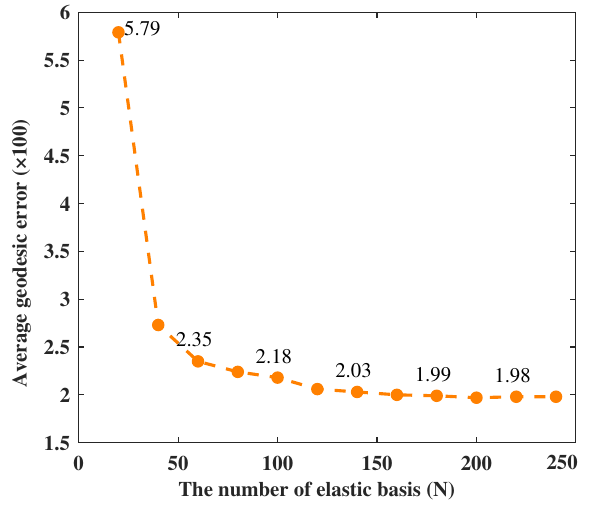}
	\end{centering}
	\caption{Impact of the number of ELA basis functions. This plot shows the geodesic error on the SCAPE dataset as the number of ELA bases used in our refiner HWF is varied.}
	\label{num_scape}
\end{figure}

\begin{figure}[h!t]
	\begin{centering}
		\includegraphics[width=0.96\linewidth]{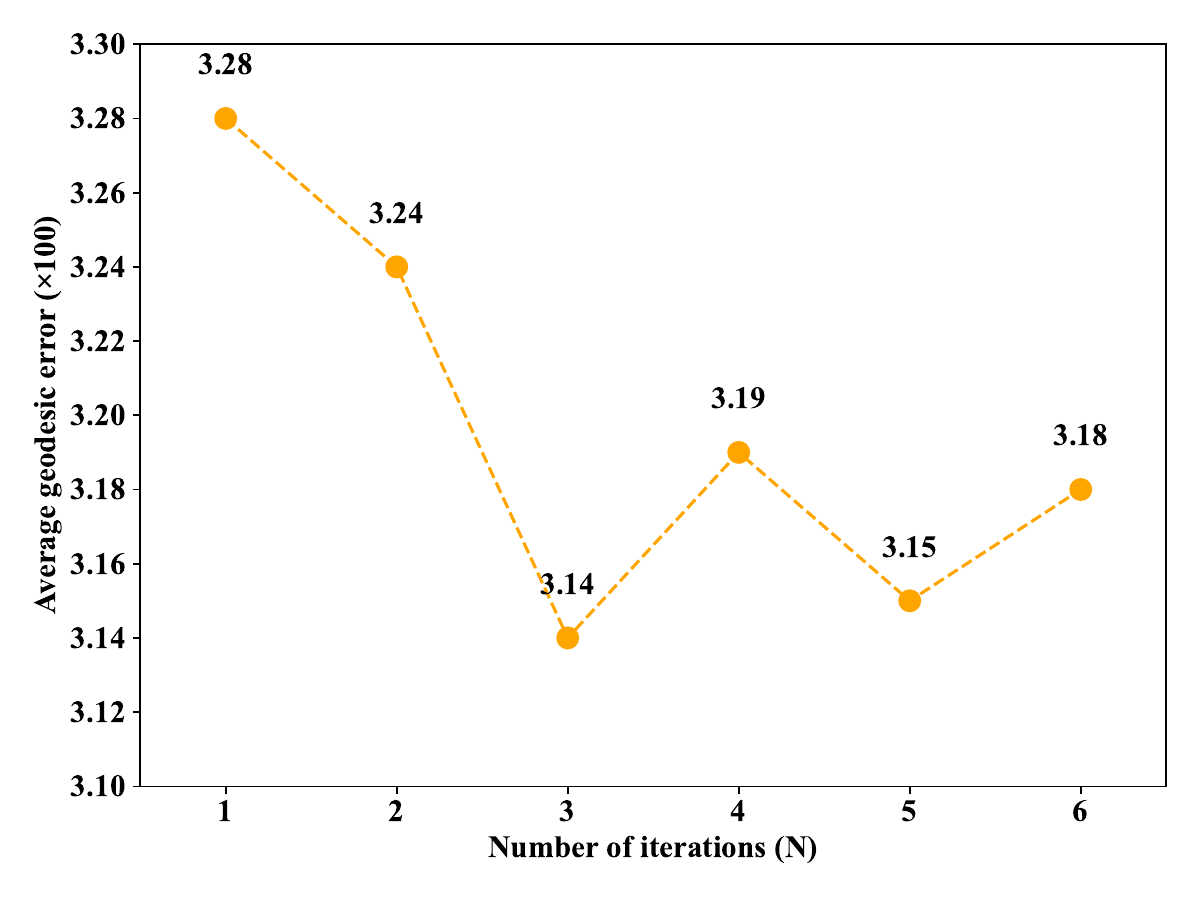}
	\end{centering}
	\caption{Convergence of the HWF refinement algorithm. This plot shows the geodesic error on the SMAL dataset as a function of the number of HWF iterations. The performance saturates after a few iterations.}
	\label{num_iter}
\end{figure}

\textbf{\textit{Number of Refinement Iterations.}} We analyzed the convergence of our HWF iterative algorithm by tracking the geodesic error on the SMAL dataset with an increasing number of iterations. The results are presented in Figure \ref{num_iter}. To balance matching accuracy with computational efficiency, we selected three refinement iterations for all experiments reported in this paper.

\textbf{\textit{Number of Wavelet Filters.}} The use of wavelet filters effectively performs frequency-domain filtering on functional maps, which is crucial for high-quality pointwise map recovery. We performed an ablation experiment on the SMAL dataset to analyze the effect of the number of wavelet filters, and the results in Figure \ref{wa} indicate that employing six filters yields the best performance. Therefore, to balance both accuracy and efficiency, we set the number of filters to six.

\begin{figure}[h!t]
	\begin{centering}
		\includegraphics[width=0.96\linewidth]{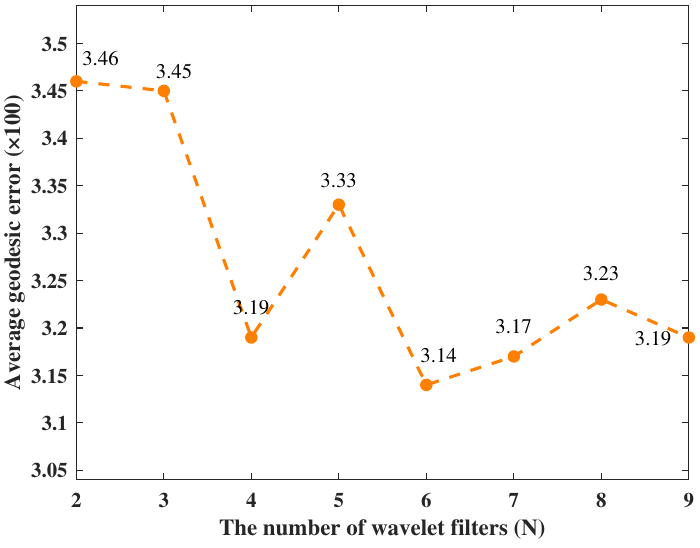}
	\end{centering}
	\caption{Ablation analysis of the number of wavelet filters. The figure illustrates the trend of average geodesic error versus the number of wavelet filters on the SMAL dataset, demonstrating the effectiveness of multi-scale filtering in the HWF refinement algorithm.}
	\label{wa}
\end{figure}

\begin{figure}[h!t]
	\begin{centering}
		\includegraphics[width=0.96\linewidth]{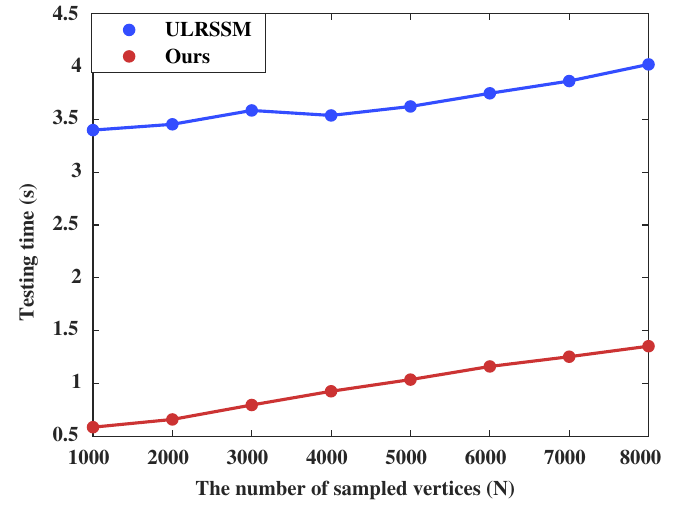}
	\end{centering}
	\caption{Comparison of inference time on the SMAL dataset. The plot shows the inference time versus the number of vertices. The blue line represents ULRSSM, which includes a test-time adaptation (TTA) post-processing step. The red line represents our method, which requires no post-processing and is significantly faster.}
	\label{TTA}
\end{figure}
\subsection{Inference Time Analysis}
Our method is designed to be efficient at test time, as it does not require any costly post-processing modules. To analyze this efficiency, we compare the inference time of our method against ULRSSM \cite{Cao2023Unsupervised}, a state-of-the-art approach that employs a costly test-time adaptation (TTA) post-processing step. In our framework, the final pointwise map is obtained directly from the non-differentiable branch during a single forward pass. Figure \ref{TTA} plots the inference time as a function of the number of vertices. The results clearly show that our method is not only highly accurate but also significantly more efficient at test time.

\clearpage
\begin{figure*}[h!t]
	\begin{centering}
		\includegraphics[width=0.96\linewidth]{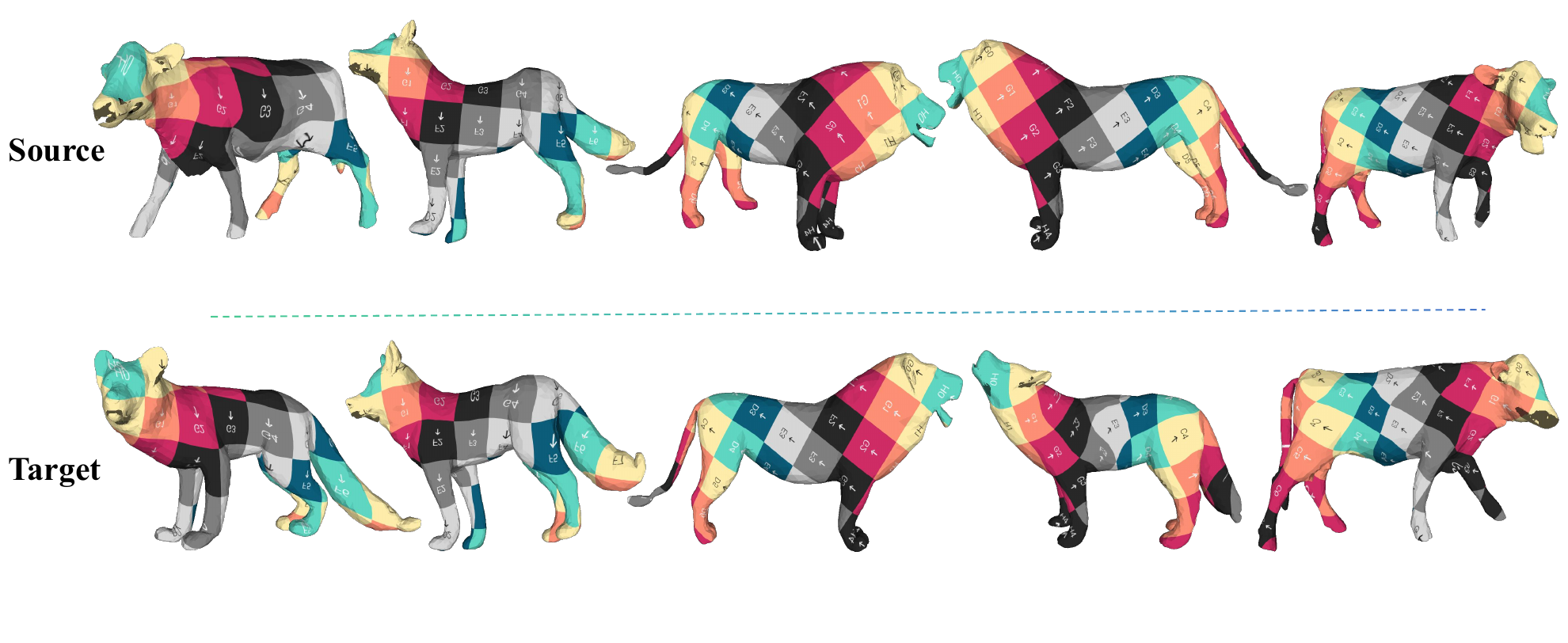}
	\end{centering}
	\caption{Qualitative results on the non-isometric SMAL benchmark. The quality of the texture transfer highlights the accuracy of the correspondence maps generated by our method on these challenging animal shapes.}
	\label{smal_tex}
\end{figure*}

\begin{figure*}[h!t]
	\begin{centering}
		\includegraphics[width=0.96\linewidth]{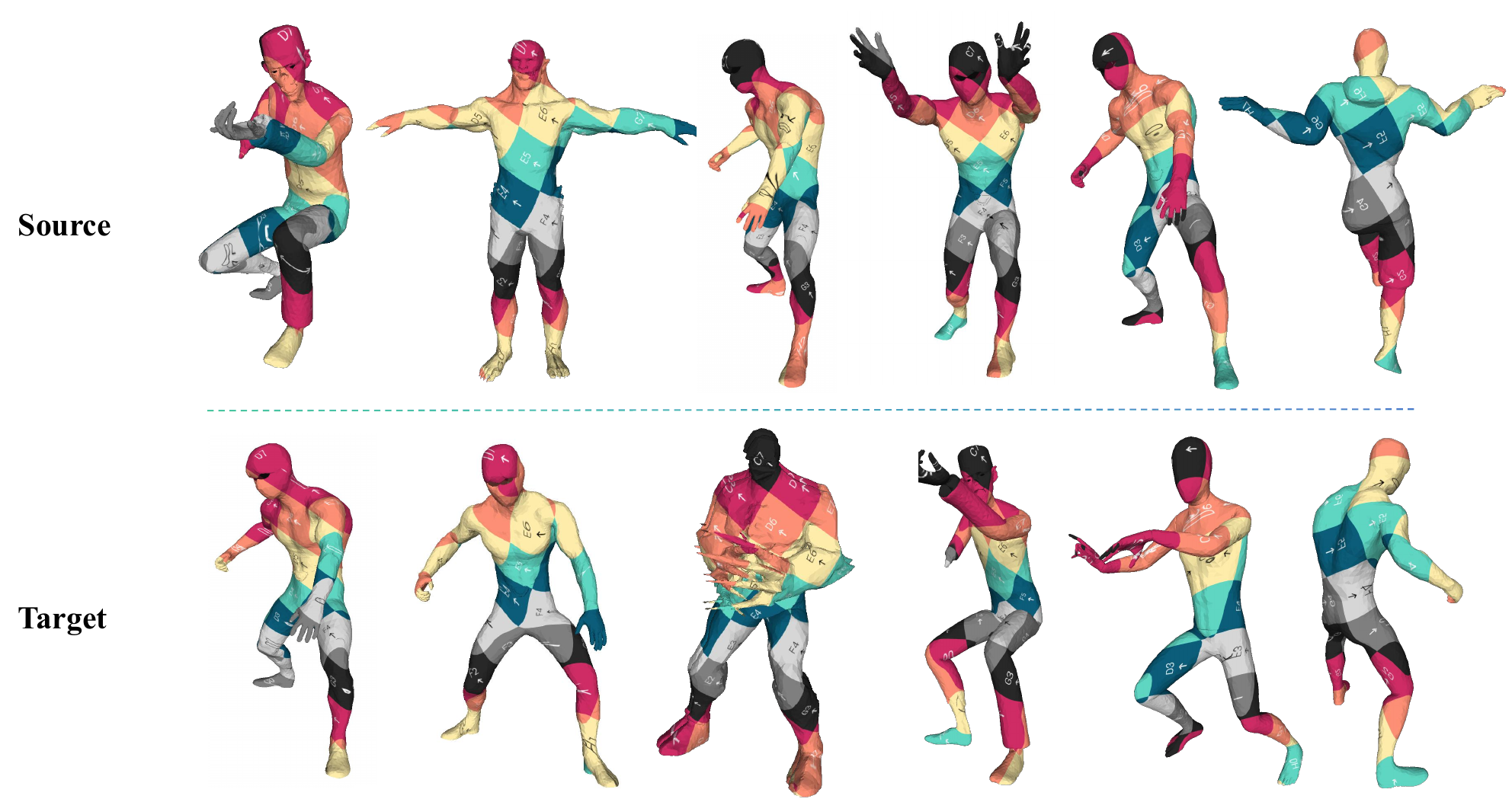}
	\end{centering}
	\caption{Matching results on the DT4D-H dataset. The quality of the texture transfer demonstrates our method's ability to find meaningful correspondences between different classes of humanoid shapes.}
	\label{DT4Dinter}
\end{figure*}
\begin{figure*}[h!t]
	\begin{centering}
		\includegraphics[width=0.96\linewidth]{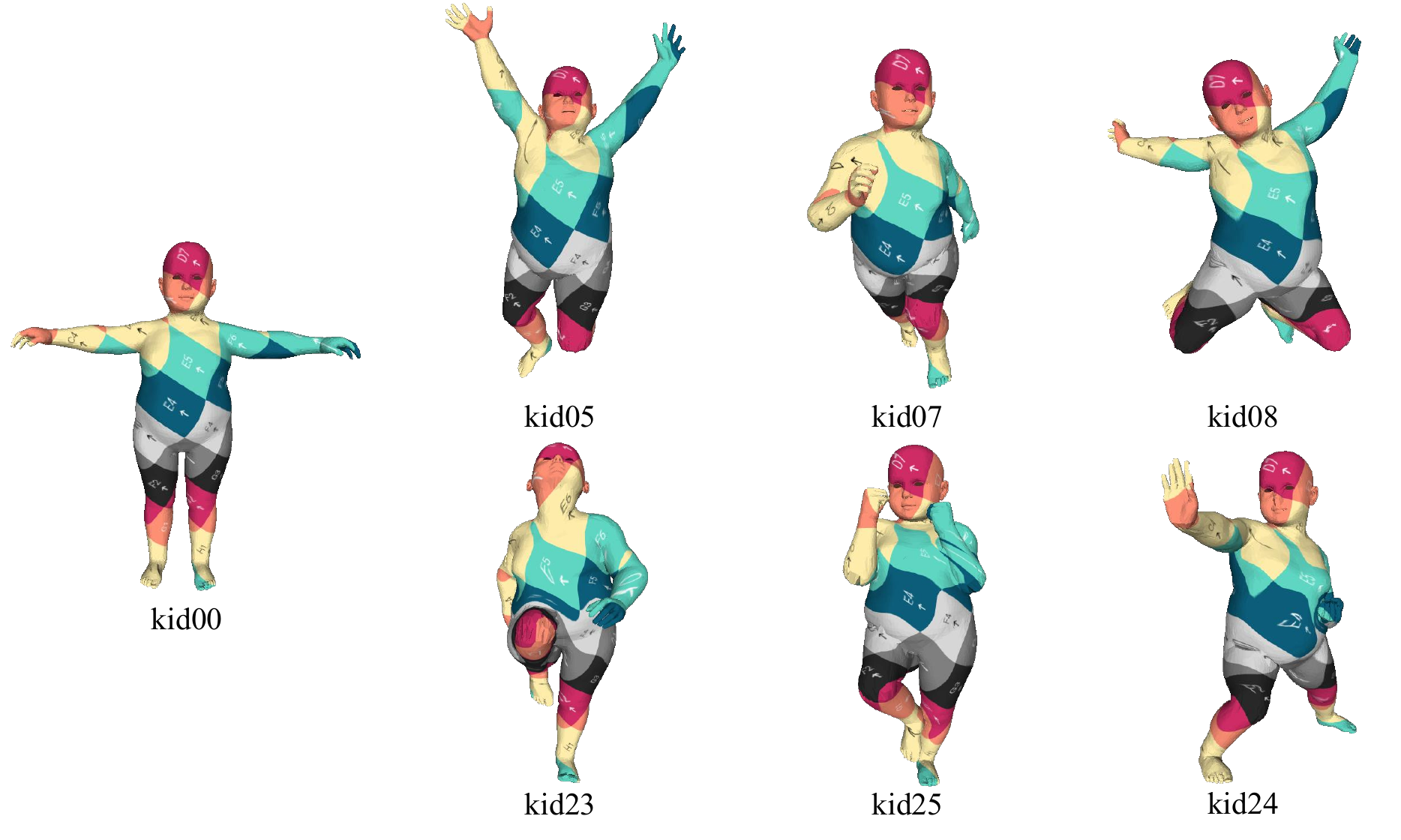}
	\end{centering}
	\caption{Robustness to significant topological noise on the TOPKIDS dataset. The figure visualizes correspondence quality via texture transfer. The leftmost shape (kid00) serves as the source, mapped onto the various target shapes. Our method successfully handles the severe topological artifacts present in the data.}
	\label{topkids}
\end{figure*}

\end{document}